\documentclass[article]{IEEEtran}

\usepackage{amsmath}  
\usepackage{tabularx} 
\usepackage{graphicx} 
\usepackage{hyperref}
\usepackage[symbols]{}
\usepackage[epstopdf]{}
\usepackage{amsthm} 
\usepackage{amssymb}
\usepackage{color}
\usepackage{mathtools}
\usepackage{commath}
\usepackage{mathtools}
\usepackage{algorithm}
\usepackage{algpseudocode}
\usepackage{cleveref}
\usepackage{cite} 
\usepackage{soul,comment,xcolor} 
\usepackage{subfiles}
\usepackage{cancel}
\usepackage{booktabs}

\newcommand{\Expc}{\mathbb{E}}

\newtheoremstyle{newstyle}
{} 
{} 
{\mdseries} 
{} 
{\bfseries} 
{.} 
{ } 
{} 

\theoremstyle{newstyle}

\newtheorem{theorem}{Theorem}
\newtheorem{lemma}{Lemma}

\theoremstyle{definition}
\newtheorem{example}{Example}

\newtheorem{assumption}{Assumption}

\theoremstyle{remark}
\newtheorem{remark}{Remark}

\newcommand{\gr}{\nabla} 
\newcommand{\al}{\alpha} 
\newcommand{\eps}{\varepsilon}
\newcommand{\bx}{\mathbf{x}}
\newcommand{\bW}{\mathbf{W}}
\newcommand{\bI}{\mathbf{I}}
\newcommand{\bA}{\mathbf{A}}

\newcommand{\bb}{\mathbf{b}}
\newcommand{\bz}{\mathbf{z}}
\newcommand{\bu}{\mathbf{u}}
\newcommand{\bE}{\mathbb{E}}
\newcommand{\cF}{\mathcal{F}}
\newcommand{\cO}{\mathcal{O}}

\begin{document}

\title{\LARGE \bf
An Exact Quantized Decentralized Gradient Descent Algorithm 
}

\author{Amirhossein Reisizadeh, Aryan Mokhtari, Hamed Hassani, Ramtin Pedarsani
\thanks{Amirhossein Reisizadeh and Ramtin Pedarsani are with the Department of Electrical and Computer Engineering at 
        University of California, Santa Barbara
        {\tt\small reisizadeh@ucsb.edu,ramtin@ece.ucsb.edu}}%
\thanks{Aryan Mokhtari is with the Department of Electrical and Computer Engineering at the University of Texas at Austin
        {\tt\small mokhtari@austin.utexas.edu}}%
\thanks{Hamed Hassani is with the Department of Electrical and Systems Engineering at University of Pennsylvania
        {\tt\small hassani@seas.upenn.edu}}
\thanks{This work is partially supported by NSF grant CCF-1755808 and the UC Office of President under grant No. LFR-18-548175. The research of H. Hassani is supported by NSF grants 1755707  and 1837253.}
\thanks{A preliminary version of this work is published in the proceedings of the 57th IEEE Conference on Decision and Control, 2018 \cite{reisizadeh2018quantized}.}
}

\maketitle

\begin{abstract}
We consider the problem of decentralized consensus optimization, where the sum of $n$ smooth and strongly convex functions are minimized over $n$ distributed agents that form a connected network. In particular, we consider the case that the communicated local decision variables among nodes are quantized in order to alleviate the communication bottleneck in distributed optimization. We propose the Quantized Decentralized Gradient Descent (QDGD) algorithm, in which nodes update their local decision variables by combining the quantized information received from their neighbors with their local information. We prove that under standard strong convexity and
smoothness assumptions for the objective function, QDGD achieves a vanishing mean solution error under customary conditions for quantizers. To the best of our knowledge, this is the first algorithm that achieves vanishing consensus error in the presence of quantization noise. Moreover, we provide simulation results that show tight agreement between our derived theoretical convergence rate and the numerical results.
\end{abstract}

\section{Introduction}
Distributed optimization of a sum of convex functions has a variety of applications in different areas including decentralized control systems \cite{cao2013overview}, wireless systems \cite{ribeiro2010ergodic}, sensor networks \cite{rabbat2004distributed}, networked multiagent systems \cite{olfati2007consensus}, multirobot networks \cite{ren2007information}, and large scale machine learning \cite{tsianos2012consensus}. In such problems, one aims to solve a consensus optimization problem to minimize $f(\bx) = \sum_{i=1}^{n} f_i(\bx)$
cooperatively over $n$ nodes or agents that form a connected network. The function $f_i( \cdot )$ represents the local cost function of node $i$ that is only known by this node.

Distributed optimization has been largely studied in the literature starting from seminal works in the 80s \cite{tsitsiklis1986distributed,tsitsiklis1984problems}. Since then, various algorithms have been proposed to address decentralized consensus optimization in multiagent systems. The most commonly used algorithms are decentralized gradient descent or gradient projection method \cite{nedic2009distributed,yuan2016convergence,jakovetic2014fast,ram2010distributed}, distributed alternating direction method of multipliers (ADMM) \cite{boyd2011distributed,shi2014linear,mokhtari64dqm},  decentralized dual averaging \cite{duchi2012dual,tsianos2012push}, and distributed Newton-type methods \cite{wei2013distributed,mokhtari2017network,eisen2017decentralized}. Furthermore, the decentralized consensus optimization problem has been considered in online or dynamic settings, where the dynamic cost function becomes an online regret function \cite{yan2013distributed}.

A major bottleneck in achieving fast convergence in decentralized consensus optimization is limited communication bandwidth among nodes. As the dimension of input data increases (which is the current trend in large-scale distributed machine learning),  a considerable amount of information must be exchanged among nodes, over many iterations of the consensus algorithm. This causes a significant communication bottleneck that can substantially slow down the convergence time of the algorithm \cite{seide20141,chowdhury2011managing}.

Quantized communication for the agents is brought into the picture for bounded and stable control systems \cite{yuksel2003quantization}. Furthermore, consensus distributed averaging algorithms are studied under discretized message passing \cite{Kashyap2006QuantizedC}. Motivated by the energy and bandwidth-constrained wireless sensor networks, the work in \cite{rabbat2005quantized} proposes distributed optimization algorithms under quantized variables and guarantees convergence within a non-vanishing error. Deterministic quantization has been considered in distributed averaging algorithms \cite{el2016design} where the iterations converge to a neighborhood of the average of initials. However, randomized quantization schemes are shown to achieve the average of initials, in expectation \cite{aysal2007distributed}. The work in \cite{nedic2008distributed} also considers a consensus distributed optimization problem over a cooperative network of agents restricted to quantized communication. The proposed algorithm guarantees convergence to the optima within an error which depends on the network size and the number of quantization levels. Aligned with the  communication bottleneck described earlier, \cite{gravelle2014quantized} provides a quantized distributed load balancing scheme that converges to a set of desired states while the nodes are constrained to remain under maximum load capacities. 

More recently, 1-Bit SGD \cite{seide20141} was introduced in which at each time step, the agents sequentially quantize their local gradient vectors by entry-wise signs while contributing the quantization error induced in previous iteration. Moreover, in \cite{alistarh2017qsgd}, the authors propose the Quantized-SGD (QSGD), a class of compression scheme algorithms that is based on a stochastic and unbiased quantizer of the vector to be transmitted. QSGD provably provides convergence guarantees, as well a good practical performance. Recently, a different line of work has proposed the use of coding theoretic techniques to alleviate the communication bottleneck in distributed computation \cite{li2016fundamental,lee2016speeding,ezzeldin2017communication,prakash2018coded}. In particular, distributed computing algorithms such as MapReduce require shuffling of data or messages between different phases of computation that incur large communication overhead. The key idea to reducing this communication load is to exploit excess in storage and local computation so that \emph{coded} messages can be sent in the phase of shuffling for reducing the communication load.

In this paper, our goal is to analyze the quantized decentralized consensus optimization problem, where node $i$ transmits a quantized version of its local decision variable $Q(\bx_i)$ to the neighboring nodes instead of the exact decision variable $\bx_i$. Motivated by the stochastic quantizer proposed in \cite{alistarh2017qsgd}, we consider two classes of unbiased random quantizers. While they both share the unbiasedness assumption, i.e. $\mathbb{E} \left[Q(\bx) | \bx \right] = \bx$, the corresponding variance differs for the two classes. We firstly consider variance bounded quantizers in which we have $\mathbb{E} \left[\| Q(\bx) - \bx\|^2 | \bx \right] \leq \sigma^2$ for some fixed constant $\sigma^2$. Furthermore, we consider random quantizers for which the variance is bounded proportionally to the norm squared of the quatizer's input, that is $\bE \left[ \norm{Q(\bx)-\bx}^2 | \bx \right]\leq \eta^2 \norm{\bx}^2 $ for a constant $\eta^2$.

 Our main contribution is to propose a Quantized Decentralized Gradient Descent (QDGD) method, which involves a novel way of updating the local decision variables by combining the quantized message received from the neighbors and the local information such that proper averaging is performed over the local decision variable and the neighbors' quantized vectors. We prove that under standard strong convexity and smoothness assumptions, for any unbiased and variance bounded quantizer, QDGD achieves a vanishing mean solution error: for all nodes $i=1,\dots,n$ we obtain that for any arbitrary $\delta \in (0,1/2)$ and large enough $T$, $\mathbb{E} \left[ \norm{ \bx_{i,T} - \widetilde{\bx}^*}^2 \right] \leq \mathcal{O} \left( \frac{1}{T^{\delta}} \right)$, where  $\bx_{i,T}$ is the local decision variable of node $i$ at iteration $T$ and $\widetilde{\bx}^*$ is the global optimum. To the best of our knowledge, this is the first decentralized gradient-based algorithm that achieves vanishing consensus error in the presence of non-vanishing quantization noise. We further generalize the convergence result to the second class of unbiased quantizers for which the variance is bounded proportionally to the norm squared of the quatizer's input and prove that the propsoed algorithm attains the same convergence rate. We also provide simulation results -- for both synthetic and real data -- that corroborate our theoretical results.
\vspace{2mm}
\\
\textbf{Notation.} In this paper, we denote by $[n]$ the set $\{1,\cdots,n\}$ for any natural number $n \in \mathbb{N}$. The gradient of a function $f(\bx)$ is denoted by $\nabla f(\bx)$. For non-negative functions $g$ and $h$ of $t$, we denote $g(t)=\mathcal{O}(h(t))$ if there exist $t_0 \in \mathbb{N}$ and constant $c$ such that  $g(t) \leq ch(t)$ for any $t \geq t_0$. We use $\lceil x \rceil$ to indicate the least integer greater than or equal to $x$.
\vspace{2mm}
\\
\textbf{Paper Organization.} The rest of the paper is organized as follows. In Section \ref{sec:problem}, we precisely formulate the quantized decentralized consensus optimization problem. We provide the description of the Quantized Decentralized Gradient Descent algorithm in Section \ref{sec:algorithm}. The main theorems of the paper are stated and proved in Section \ref{sec:convergence}. In Section \ref{sec:extension}, we study the trade-off between communication cost and accuracy of the algorithm. We provide numerical studies in Section \ref{sec:simulation}. Finally, we conclude the paper and discuss future directions in Section \ref{sec:conclusion}.


\section{Problem Formulation}\label{sec:problem}
In this section, we formally define the consensus optimization problem that we aim to solve. Consider a set of $n$ nodes that communicate over a connected and undirected graph $\mathcal{G}=(\mathcal{V},\mathcal{E})$ where $\mathcal{V}=\{1,\cdots,n\}$ and $\mathcal{E} \subseteq \mathcal{V} \times \mathcal{V}$ denote the set of nodes and edges, respectively. We assume that nodes are only allowed to exchange information with their neighbors and use the notation $\mathcal{N}_i$ for the set of node $i$'s neighbors. In our setting, we assume that each node $i$ has access to a local convex function $f_{i}: \mathbb{R}^p\to  \mathbb{R}$, and nodes in the network cooperate to minimize the aggregate objective function $f: \mathbb{R}^p\to  \mathbb{R}$ taking values $f(\bx)=\sum_{i=1}^nf_{i}(\bx)$. In other words, nodes aim to solve the optimization problem
\begin{equation} \label{eq:main}
     \operatorname*{\text{min}}_{\bx \in \mathbb{R}^p}\ f(\bx) \ =\ \operatorname*{\text{min}}_{\bx \in \mathbb{R}^p}\ \sum_{i=1}^{n} f_i(\bx).
\end{equation}
We assume the local objective functions $f_i$ are strongly convex and smooth, and, therefore, the aggregate function $f$ is also strongly convex and smooth. In the rest of the paper, we use $\widetilde{\bx}^*$ to denote the unique minimizer of Problem~\eqref{eq:main}.

In decentralized settings, nodes have access to a single summand of the global objective function $f$ and to reach the optimal solution $\widetilde{\bx}^*$, communication with neighboring nodes is inevitable. To be more precise, nodes need to minimize their local objective functions, while they ensure that their local decision variables are equal to their neighbors'. This interpretation leads to an equivalent formulation of Problem~\eqref{eq:main}. If we define $\bx_i$ as the decision variable of node~$i$, the alternative formulation of Problem~\eqref{eq:main} can be written as 
\begin{align} \label{eq:main_2}
   & \operatorname*{\text{min}}_{\bx_1,\dots,\bx_n \in \mathbb{R}^p}\ \sum_{i=1}^{n} f_i(\bx_i) \nonumber\\
   & \ \text{subject to} \quad \bx_i=\bx_j,\qquad  \text{for all}\ i, \ j\in \mathcal{N}_i.
\end{align}
Since we assume that the underlying network is a connected graph, the constraint in \eqref{eq:main_2} implies that any feasible solution should satisfy $\bx_1=\dots=\bx_n$. Under this condition the objective function values in \eqref{eq:main} and \eqref{eq:main_2} are equivalent. Hence, it follows that the optimal solutions of Problem \eqref{eq:main_2} are equal to the optimal solution of Problem \eqref{eq:main}, i.e., if we denote $\{\bx_i^*\}_{i=1}^n$ as the optimal solutions of Problem \eqref{eq:main_2} it holds that $\bx_1^*=\dots=\bx_n^*=\widetilde{\bx}^*$. Therefore, we proceed to solve Problem \eqref{eq:main_2} which is naturally formulated for decentralized optimization in lieu of Problem \eqref{eq:main}. 

The problem formulation in \eqref{eq:main_2} suggests that each node $i$ should minimize its local objective function $f_i$ while keeping its decision variable $\bx_i$ close to the decision variable $\bx_j$ of its neighbors $j\in\mathcal{N}_i$. This goal can be achieved by exchanging local variables $\bx_i$ among neighboring nodes to enforce consensus on the decision variables. Indeed, exchange of updated local vectors between the distributed nodes induces a potentially heavy communication load on the shared bus. To address this issue, we assume that each node provides a randomly quantized variant of its local updated variable to the neighboring nodes. That is, if we denote by $\bx_i$ the decision variable of node $i$, then the corresponding quantized variant $\bz_i=Q(\bx_i)$ is communicated to the neighboring nodes, $\mathcal{N}_i$. Exchanging quantized vectors $\bz_i$ instead of the true vectors $\bx_i$ indeed reduces the communication burden at the cost of injecting noise to the information received by the nodes in the network. The main challenge in this setting is to ensure that nodes can still converge to the optimal solution of Problem~\eqref{eq:main_2}, while they only have access to a quantized variant of their neighbors' true decision variables.

\section{QDGD Algorithm}\label{sec:algorithm}
In this section, we propose a quantized gradient based method to solve the decentralized optimization problem in \eqref{eq:main_2} and consequently the original problem in \eqref{eq:main} in a fully decentralized fashion. To do so, consider $\bx_{i,t}$ as the decision variable of node $i$ at step $t$ and  $\bz_{i,t} = Q(\bx_{i,t})$ as the quantized version of the vector $\bx_{i,t}$. In the proposed Quantized Decentralized Gradient Descent (QDGD) method, nodes update their local decision variables by combining the quantized information received from their neighbors with their local information. To formally state the update of QDGD, we first define $w_{ij}$ as the weight that node $i$ assigns to node $j$. If nodes $i$ and $j$ are not neighbors then $w_{ij}=0$, and if they are neighbors the weight $w_{ij}\geq0$ is nonnegative. At each time step $t$, each node $i$ sends its quantized $\bz_{i,t}$ variant of its local vector $\bx_{i,t} $ to its neighbors $j\in \mathcal{N}_i$ and receives their corresponding vectors $\bz_{j,t}$. Then, using the received information it updates its local decision variable according to the update  
\begin{equation}\label{eq:alg_update}
    \bx_{i,t+1} =  (1 - \eps + \eps w_{ii}) \bx_{i,t} + \eps \sum_{ j \in  \mathcal{N}_i } w_{ij} \bz_{j,t} - \al \eps \gr f_i(\bx_{i,t}),
\end{equation}
where $\eps$ and $\alpha$ are positive step-sizes.

The update of QDGD in \eqref{eq:alg_update} shows that the updated iterate is a linear combination of the weighted average of node $i$'s  neighbors' decision variable, i.e.,  $\eps\sum_{ j \in  \mathcal{N}_i } w_{ij} \bz_{j,t}$, and its local variable $\bx_{i,t}$ and  gradient ${\gr} f_i(\bx_{i,t})$. The parameter $\alpha$ behaves as the stepsize of the gradient descent step with respect to local objective function and the parameter $\eps$ behaves as an averaging parameter between performing the distributed gradient update  $\eps( w_{ii}\bx_{i,t} +  \sum_{ j \in  \mathcal{N}_i } w_{ij} \bz_{j,t} - \al  {\gr} f_i(\bx_{i,t}))$ and using the previous decision variable $(1-\eps)\bx_{i,t}$. By choosing a diminishing stepsize $\alpha$ and averaging using the parameter $\eps$ we control randomness induced by exchanging quantized variables. The steps of the proposed QDGD method are summarized in Algorithm \ref{alg:QDC}.

\begin{algorithm}[t]
\caption{QDGD at node $i$}\label{alg:QDC}
\begin{algorithmic}[1]
\Require Weights $\{w_{ij}\}_{j=1}^n$, total iterations $T$
\State Set $\bx_{i,0}=0$ and compute $\bz_{i,0}=Q(\bx_{i,0})$ 
\For{$t=0,\cdots,T-1$}
	\State Send $\bz_{i,t}=Q(\bx_{i,t})$ to $j\in\mathcal{N}_i$ and receive $\bz_{j,t}$
        \State Compute  $\bx_{i,t+1}$ according to the update in \eqref{eq:alg_update}
\EndFor
\State \Return $\bx_{i,T}$
\end{algorithmic}
\end{algorithm}

\begin{remark}
The proposed QDGD algorithm can be interpreted as a variant of the decentralized (sub)gradient descent (DGD) method \cite{nedic2009distributed,yuan2016convergence} for quantized decentralized optimization (see Section \ref{sec:convergence}). Note that the vanilla DGD method converges to a neighborhood of the optimal solution in the presence of quantization noise where the radius of convergence depends on the variance of quantization error \cite{nedic2009distributed,yuan2016convergence,rabbat2005quantized,nedic2008distributed}. QDGD improves the inexact convergence of quantized DGD by modifying the contribution of quantized information received from neighboring noise as described in update~\eqref{eq:alg_update}. In particular, as we show in Theorem \ref{thm1}, the sequence of iterates generated by QDGD converges to the optimal solution of Problem \eqref{eq:main} in expectation.
\end{remark}

Note that the proposed QDGD algorithm does not restrict the quantizer, except for few customary conditions. However, design of efficient quantizers has been taken into consideration. Consider the following example as such quantizers.

\begin{example} \label{ex:q1}
Consider a low-precision representation specified by $\gamma \in \mathbb{R}$ and $b \in \mathbb{N}$. The range representable by scale factor $\gamma$ and $b$ bits is
$ \{ -\gamma \cdot 2^{b-1}, \cdots, -\gamma, 0, \gamma, \cdots, \gamma \cdot (2^b-1) \}$. For any $k\gamma \leq x < (k+1)\gamma$ in the representable range, the low-precision quantizer outputs
\begin{equation}
    Q_{(\gamma,b)}(x) =
    \left\{
	\begin{array}{ll}
		k\gamma  & \mbox{w.p. } 1-\frac{x-k\gamma}{\gamma}, \\
		(k+1)\gamma & \mbox{w.p. } \frac{x-k\gamma}{\gamma}.
	\end{array}
    \right.
\end{equation}
For any $x$ in the range, the quantizer is unbiased and variance bounded, i.e.  $\bE\left[Q_{(\gamma,b)}(x)\right]=x$ and $\bE \left[ \norm{Q_{(\gamma,b)}(x) - x}^2 \right]\leq\frac{\gamma^2}{4}$. 
\end{example}

In Section~\ref{sec:convergence}, we formally state the required conditions for the quantization scheme used in QDGD and show that a large class of well-known quantizers satisfy the required conditions. 



\section{Convergence Analysis}\label{sec:convergence}

In this section, we prove that for sufficiently large number of iterations, the sequence of local iterates generated by QDGD converges to an arbitrarily precise approximation of the optimal solution of Problem~\eqref{eq:main_2} and consequently Problem~(\ref{eq:main}). The following assumptions hold throughout the analysis of the algorithm.
\begin{assumption}\label{assump1}
Local objective functions $f_i$ are differentiable and smooth with parameter $L$, i.e.,
\begin{equation}
    \norm{\gr f_i(\bx) - \gr f_i(\mathbf{y})} \leq L \norm{\bx - \mathbf{y}}, 
\end{equation}
for any $\bx, \mathbf{y} \in \mathbb{R}^{p}$. \footnote{Local objectives may have different smoothness parameters, however, WLOG one can consider the largest smoothness parameter as the one for all the objectives.}
\end{assumption}

\begin{assumption}\label{assump2}
Local objective functions $f_i$ are strongly convex with parameter $\mu$, i.e.,
\begin{equation}
    \langle \gr f_i(\bx) - \gr f_i(\mathbf{y}) , \bx - \mathbf{y} \rangle \geq \mu \norm{\bx - \mathbf{y}}^2, 
\end{equation}
for any $\bx, \mathbf{y} \in \mathbb{R}^{p}$.\footnote{Local objectives may have different strong convexity parameters, however, WLOG one can consider the smallest strong convexity parameter as the one for all the objectives.}

\end{assumption}

\begin{assumption}\label{assump3}
The random quantizer $Q(\cdot)$ is unbiased and has a bounded variance, i.e.,
\begin{equation}
    \bE \left[Q(\bx)| \bx \right]=\bx, \quad \text{ and } \quad \bE \left[ \norm{Q(\bx)-\bx}^2 | \bx \right]\leq \sigma^2,
\end{equation}
for any $\bx \in \mathbb{R}^{p}$; and quantizations are carried out independently on distributed nodes.
\end{assumption}

\begin{assumption}\label{assump4}
The weight matrix $W \in \mathbb{R}^{n\times n}$  with entries $w_{ij}$ satisfies the following conditions 
\begin{equation}
W = W^\top, \quad W \mathbf{1}=\mathbf{1},\quad \text{ and }  \quad \text{null}(I-W)= \text{span}(\mathbf{1}).
\end{equation}
\end{assumption}

The conditions in Assumptions \ref{assump1} and \ref{assump2} imply that the global objective function $f$ is strongly convex with parameter $\mu$ and its gradients are Lipschitz continuous with constant $L$. Assumption \ref{assump3} poses two customary conditions on the quantizer, that are unbiasedness and variance boundedness. Assumption \ref{assump4} implies that weight matrix $W$ is symmetric and  doubly stochastic. The largest eigenvalue of $W$ is $\lambda_1(W)=1$ and all the eigenvalues belong to $(-1,1]$, i.e., the ordered sequence of eigenvalues of $W$ are $1=\lambda_1(W) \geq \lambda_2(W) \geq \cdots \geq \lambda_n(W)  > -1$. We denote by $1-\beta$ the spectral gap associated to the stochastic matrix $W$, where $\beta=\max \left\{|\lambda_2(W)|,|\lambda_n(W)| \right\}$ is the second largest magnitude of the eigenvalues of matrix $W$. It is also customary to assume $\text{rank}(I-W)=n-1$ such that $\text{null}(I-W) = \text{span}(\mathbf{1})$. We let $W_D$ denote the diagonal matrix consisting of the diagonal entries of $W$, i.e. $\{w_{11},\cdots,w_{nn}\}$.

In the following theorem we show that the local iterations generated by QDGD converge to the global optima, as close as desired.

\begin{theorem}\label{thm1}
Consider the distributed consensus optimization Problem (\ref{eq:main}) and suppose Assumptions \ref{assump1}--\ref{assump4} hold. Consider $\delta$ as an arbitrary scalar in $(0,1/2)$ and set $\eps=\frac{c_1}{T^{3\delta/2}}$ and $\al = \frac{c_2}{T^{\delta/2}}$ where $c_1$ and $c_2$ are arbitrary positive constants (independent of $T$). Then, for each node $i$, the expected difference between the output of Algorithm~\ref{alg:QDC} after $T$ iterations and the solution of Problem \eqref{eq:main}, i.e. $\widetilde{\bx}^*$ is upper bounded by
\begin{align}\label{eq:main_thm_result}
\bE \Big[\norm{\bx_{i,T} - \widetilde{\bx}^*}^2\Big] &\leq \mathcal{O} \Bigg( \Bigg(  \frac{4n c^2_2 D^2 \left(3+ 2L/\mu\right)^2 }{(1-\beta)^2} \nonumber\\
 &\quad  \quad  \quad   + \frac{2 c_1 { n} \sigma^2 \norm{W\!-\!W_D}^2}{\mu c_2} \Bigg)\frac{1}{T^{\delta}} \Bigg),
\end{align}
if the total number of iterations satisfies $T\geq T_0$, where $T_0$ is a function of $\delta$, $c_1$, $c_2$, $\mu$, $L$, and $\lambda_n(W)$. Moreover,
\begin{align}
	D^2 = 2L  \sum_{i=1}^{n} \left( f_i(0) - f^*_i \right), \quad f^*_i =  \min_{\bx \in \mathbb{R}^p} f_i(\bx).
\end{align}
\end{theorem}

Theorem \ref{thm1} demonstrates that the proposed QDGD provides an approximation solution with vanishing deviation from the optimal solution, despite the fact that the quantization noise does not vanish as the number of iterations progresses. 

By the first glance at the expression in \eqref{eq:main_thm_result} one might suggest to set $\delta=1/2$ to obtain the best possible sublinear convergence rate which is $ \cO \left( \frac{1}{T^{1/2}} \right)$. However, $T_0$, which is a lower bound on the total number of iterations $T$, is an increasing function of $1/(1-2\delta)$, and by choosing $\delta$ very close to $1/2$, the total number of iterations $T$ should be very large to obtain a fast convergence rate close to $ \cO\left( \frac{1}{T^{1/2}}\right)$. Therefore, there is a trade-off between the convergence rate and the minimum number of required iterations. By setting $\delta$ close to $1/2$ we obtain a fast convergence rate but at the cost of running the algorithm for a large number of iterations, and by selecting $\delta$ close to $0$ the lower bound on the total number of iterations becomes smaller at the cost of having a slower convergence rate. {{We will illustrate this trade-off in the numerical experiments.}}

Moreover, note that the result in \eqref{eq:main_thm_result} shows a balance between the variance of quantization and the mixing matrix. To be more precise, if the variance of quantization $\sigma^2$ is small nodes should assign larger weights to their neighbors which decreases $(1-\beta)^{-2}$ and increases $\|W-W_D\|^2$. Conversely, when the variance $\sigma^2$ is large, to balance the terms in \eqref{eq:main_thm_result} nodes should assign larger weights to their local decision variables which decreases the term $\|W-W_D\|^2$ and increases $(1-\beta)^{-2}$. 

\subsection{Proof of Theorem \ref{thm1}}
To analyze the proposed QDGD method, we start by rewriting the update rule (\ref{eq:alg_update}) as follows
\begin{equation} \label{eq:rule2}
    \bx_{i,t+1} = \bx_{i,t}- \eps \Big( (1-w_{ii}) \bx_{i,t} -\sum_{ j \neq i} w_{ij} \bz_{j,t} + \al \gr f_i(\bx_{i,t}) \Big).
\end{equation}
Note that to derive the expression in \eqref{eq:rule2}, we simply use the fact that $w_{ij}=0$ when $j\notin \mathcal{N}_i$. 

The next step is to write the update \eqref{eq:rule2} in a matrix form. To do so, we define the function $F:\mathbb{R}^{np} \to \mathbb{R}$ as $F(\bx)=\sum_{i=1}^{n} f_i(\bx_{i})$ where $\bx_i \in \mathbb{R}^{p}$ and $\bx = [\bx_1;\cdots;\bx_n] \in \mathbb{R}^{np}$ is the concatenation of the local variables $\bx_i$. It is easy to verify that the gradient of the function $F$ is the concatenation of local gradients evaluated at the local variable, that is $\gr F(\bx_t)=[ \gr f_1(\bx_{1,t});\cdots;\gr f_n(\bx_{n,t})]$. We also define the matrix $\bW = W \otimes I \in \mathbb{R}^{np \times np}$ as the Kronecker product of the weight matrix $W \in \mathbb{R}^{n \times n}$ and the identity matrix $I \in \mathbb{R}^{p \times p}$. Similarly, define $\bW_D = W_D \otimes I \in \mathbb{R}^{np \times np}$, where $W_D=[w_{ii}] \in \mathbb{R}^{n \times n}$ denotes the diagonal matrix of the entries on the main diagonal of $W$. For the sake of consistency, we denote by the boldface $\bI$ the identity matrix of size $np$. According to above definitions, we can write the concatenated version of (\ref{eq:rule2}) as follows,
\begin{equation} \label{eq:rule-matrix}
    \bx_{t+1} = \bx_{t} - \eps \Big(  \big(\bI - \bW_D \big) \bx_t+\big(\bW_{D} - \bW\big) \bz_{t}  + \al \gr F(\bx_{t}) \Big).
\end{equation}


As we discussed in Section \ref{sec:problem}, the distributed consensus optimization Problem (\ref{eq:main}) can be equivalently written as Problem (\ref{eq:main_2}). The constraint in the latter restricts the feasible set to the consensus vectors, that is $\{ \bx=[\bx_1;\cdots;\bx_n] : \bx_1=\cdots=\bx_n \}$. According to the discussion on rank of the weight matrix $W$, the null space of the matrix $I-W$ is $\text{null}(I - W)=\text{span}(\mathbf{1})$. Hence, the null space of $\bI-\bW$ is the set of all consensus vectors, i.e., $\bx \in \mathbb{R}^{np}$ is feasible for Problem (\ref{eq:main_2}) if and only if $(\bI-\bW)\bx=0$, or equivalently $(\bI-\bW)^{1/2}\bx=0$. Therefore, the alternative Problem (\ref{eq:main_2}) can be compactly represented as the following linearly-constrained problem,
\begin{equation}\label{eq:Fmin}
\begin{aligned}
& \underset{ \bx \in \mathbb{R}^{np}}{\text{min}}
& & F(\bx) =  \sum_{i=1}^{n} f_i(\bx_i) \\
& \text{subject to}
& & (\bI-\bW )^{1/2}\bx=0.
\end{aligned}
\end{equation}
We denote by $\bx^* =[\widetilde{\bx}^*;\dots;\widetilde{\bx}^*]$ the unique solution to (\ref{eq:Fmin}).

Now, for given penalty parameter $\al > 0$, one can define the quadratic penalty function corresponding to the linearly constraint problem (\ref{eq:Fmin}) as follows,
\begin{equation} \label{eq:h}
    h_{\al}(\bx)= \frac{1}{2} \bx^\top \big(\bI - \bW \big) \bx + \al  F(\bx).
\end{equation}
Since $\bI - \bW$ is a positive semi-definite  matrix and $F$ is $L$-smooth and $\mu$-strongly convex, the function $h_{\al}$ is $L_{\alpha}$-smooth and $\mu_{\alpha}$-strongly convex on $\mathbb{R}^{np}$ having $L_{\al}= 1-\lambda_n(W) + \al L $ and $\mu_{\al} = \al \mu$. We denote by $\bx^*_{\al}$ the unique minimizer of $h_{\al}(\bx)$, i.e.,
\begin{equation} \label{eq:hmin}
    \bx^*_{\al} = \operatorname*{\text{arg}\,\text{min}}_{\bx \in \mathbb{R}^{np}} h_{\al}(\bx) = \operatorname*{\text{arg}\,\text{min}}_{\bx \in \mathbb{R}^{np}} \frac{1}{2} \bx^\top \big(\bI - \bW \big) \bx + \al  F(\bx).
\end{equation}

In the following, we link the solution of Problem (\ref{eq:hmin}) to the local variable iterations provided by Algorithm \ref{alg:QDC}. Specifically, for sufficiently large number of iterations $T$, we demonstrate that for proper choice of step-sizes, the expected squared deviation of $\bx_{T}$ from $\bx^*_{\al}$ vanishes sub-linearly. This result follows from the fact that the expected value of the descent direction in \eqref{eq:rule-matrix} is an unbiased estimator of the gradient of the function $ h_{\al}(\bx)$. 

 \begin{lemma} \label{lemma1}
 Consider the optimization Problem (\ref{eq:hmin}) and suppose Assumptions \ref{assump1}--\ref{assump4} hold. Then, the expected deviation of the output of QDGD from the solution to Problem (\ref{eq:hmin}) is upper bounded by
 \begin{equation}\label{eq:lemma1}
 \Expc \Big[\norm{\bx_T - \bx^*_{\al}}^2 \Big] \leq \mathcal{O} \left( \frac{c_1 { n} \sigma^2 \norm{W\!-\!W_D}^2}{\mu c_2} \frac{1}{T^{\delta}} \right),
\end{equation}
for $\eps=\frac{c_1}{T^{3\delta/2}}$, $\al = \frac{c_2}{T^{\delta/2}}$, any $\delta \in (0,1/2)$ and $T \geq T_1$, where $c_1$ and $c_2$ are positive constants independent of $T$, and 
\begin{equation}
 T_1 \coloneqq \text{max} \left\{  e^{e^{\frac{1}{1-2\delta}}},  \left\lceil \left( c_1 c_2 \mu \right)^{\frac{1}{2\delta}} \right\rceil , \left\lceil \left( \frac{c_1 (2+c_2 L)^2}{c_2 \mu} \right)^{\frac{1}{\delta}} \right\rceil \right\}. 
 \end{equation}
 \end{lemma}
\vspace{0.5mm}
 \begin{proof}
 See Appendix \ref{app_proof_of_lemma_1}.
 \end{proof}

Lemma \ref{lemma1} guarantees convergence of the proposed iterations according to the update in (\ref{eq:alg_update}) to the solution of the later-defined Problem (\ref{eq:hmin}). Loosely speaking, Lemma \ref{lemma1} ensures that $\bx_T$ is \textit{close} to $\bx^*_{\al}$ for large $T$. So, in order to capture the deviation of $\bx_T$ from the global optima $\bx^*$, it suffices to show that $\bx^*_{\al}$ is \textit{close} to $\bx^*$, as well. As the problem in \eqref{eq:hmin} is a penalized version of the original constrained program in \eqref{eq:main}, the solutions to these two problems should not be significantly different if the penalty coefficient $\alpha$ is small. We formalize this claim in the following lemma.

\begin{lemma}\label{lemma2}
Consider the distributed consensus optimization Problem (\ref{eq:main}) and the problem defined in (\ref{eq:hmin}). If Assumptions \ref{assump1}, \ref{assump2} and \ref{assump4} hold, then the difference between the optimal solutions to \eqref{eq:Fmin} and its penalized version \eqref{eq:hmin} is bounded above by
\begin{equation}  
\norm{\bx^*_{\al} - \bx^*} \leq  \cO \left( \frac{\sqrt{2n} c_2 D \left(3+ 2L/\mu\right)}{ 1-\beta}  \frac{1}{T^{\delta/2}} \right),
\end{equation}
for $\al=\frac{c_2}{T^{\delta/2}}$ and $T\geq T_2$, where $c_2$ is a positive constant independent of $T$, $\delta \in (0,1/2)$ is an arbitrary constant, and 
\begin{equation} 
T_2 \coloneqq \text{max} \left\{ \left\lceil \left(\frac{c_2 L}{1+\lambda_n(W)}\right)^{\frac{2}{\delta}} \right\rceil , \left\lceil c^4_2 (\mu + L)^{\frac{2}{\delta}} \right\rceil  \right\}.
\end{equation}
\end{lemma}
\vspace{0.5mm}
 \begin{proof}
 See Appendix \ref{app_proof_of_lemma_2}.
 \end{proof}

The result in Lemma \ref{lemma2} shows that if we set the penalty coefficient $\alpha$ small enough, i.e., $\al=\mathcal{O}(T^{-\delta/2})$, then the distance between the optimal solutions of the constrained problem in \eqref{eq:main} and the penalized problem in \eqref{eq:hmin} is of $\mathcal{O} \left(\frac{\alpha}{1-\beta} \right) $. 

Having set the main lemmas, now it is straightforward to prove the claim of Theorem \ref{thm1}. For the specified step-sizes $\eps$ and $\al$ and large enough iterations $T \geq T_0 \coloneqq \text{max} \left\{ T_1, T_2 \right\}$, Lemmas \ref{lemma1} and \ref{lemma2} are applicable and we have
\begin{align}\label{eq:result}
    \bE \left[\norm{\bx_{T} - \bx^*}^2\right] &= \bE \left[\norm{\bx_T - \bx^*_{\al} + \bx^*_{\al} - \bx^*}^2\right] \nonumber \\
    &\leq 2 \Expc \left[\norm{\bx_T - \bx^*_{\al}}^2 \right] + 2 \norm{\bx^*_{\al} - \bx^*}^2 \nonumber\\
    &\leq \cO\left(\frac{1}{T^{\delta}} \right) + \cO\left(\frac{1}{T^{\delta}} \right) \nonumber \\
    &= \cO\left(\frac{1}{T^{\delta}} \right),
\end{align}
where we used $\norm{\mathbf{a+b}}^2 \leq 2 \big( \norm{\mathbf{a}}^2 + \norm{\mathbf{b}}^2 \big)$ to derive the first inequality; and the constants can be found in the proofs of the two lemmas. Since $ \bE \Big[\norm{\bx_{i,T} - \widetilde{\bx}^*}^2\Big]\leq  \bE \Big[\norm{\bx_{T} - \bx^*}^2\Big]$ for any $i=1,\dots,n$, the inequality in \eqref{eq:result} implies the claim of Theorem~\ref{thm1}.

 \subsection{Extension to more quantizers}
Based on the condition in Assumption~\ref{assump3}, so far we have been considering only unbiased quantizers for which the variance of quantization is bounded by a constant scalar, i.e., $\bE \left[ \|Q(\bx)-\bx\|^2 | \bx \right]\leq \sigma^2$. However, there are widely used representative quantizers where the quantization noise induced on the input is bounded proportionally to the input's magnitude, i.e., $\bE \left[ \|Q(\bx)-\bx\|^2 | \bx \right]\leq \mathcal{O} \left(\|\bx\|^2\right)$ \cite{alistarh2017qsgd}. 

Indeed, this condition is more challenging since the set of iterates norm $\|\bx_t\|$ are not necessarily bounded, and we cannot uniformly bound the variance of  the noise induced by quantization. In this subsection, we show that the proposed algorithm is converging with the same rate for quantizers satisfying this new assumption. Let us first formally state this assumption.

\begin{assumption}\label{assump5}
The random quantizer $Q(\cdot)$ is unbiased and its variance is proportionally bounded by the input's squared norm, that is,
\begin{equation}
    \bE \left[Q(\bx)| \bx \right]=\bx, \quad \text{ and } \quad \bE \left[ \norm{Q(\bx)-\bx}^2 | \bx \right]\leq \eta^2 \norm{\bx}^2,
\end{equation}
for a constant $\eta^2$ and any $\bx \in \mathbb{R}^{p}$; and quantizations are carried out independently on distributed nodes.
\end{assumption}

Before characterizing the convergence properties of the proposed QDGD method under the conditions in Assumption~\ref{assump5}, let us review a subset of quantizers that satisfy this condition.

\begin{example}[Low-precision quantizer] \label{ex:lp}
Consider the low precision quantizer $ Q^{\text{LP}}: \mathbb{R}^p \rightarrow \mathbb{R}^p$ which is defined as 
\begin{equation}\label{eq:low_per_def}
    Q^{\text{LP}}_i (\bx) = \norm{\bx} \cdot \text{sign}(x_i) \cdot \xi_i(\bx,s),
\end{equation}
where $ \xi_i(\bx,s) $ is a random variable defined as
\begin{equation}\label{eq:low_per_def_2}
    \xi_i(\bx,s) =
\left\{
	\begin{array}{ll}
	\frac{l}{ s}  &\quad \mbox{w.p.} \ \ 1- q\left(\frac{|x_i|}{\norm{\bx}},s\right), \vspace{2mm} \\ 
		\frac{l+1}{s} & \quad \mbox{w.p.} \ \ q\left(\frac{|x_i|}{\norm{\bx}},s\right), 
	\end{array}
\right.
\end{equation}
and $q(a,s)=as-l$ for any $a  \in [0,1]$. In above, the tuning parameter $s$ corresponds to the number of quantization levels and $l \in [0,s)$ is an integer such that $|x_i|/\norm{\bx} \in [l/s, (l+1)/s]$. It is not hard to check that \cite{alistarh2017qsgd}  the low precision quantizer $Q^{\text{LP}}$ defined in \eqref{eq:low_per_def} is an unbiased estimator of the vector $\bx$ and the variance is bounded above by
\begin{equation}\label{eq:low_prec_var}
    \mathbb{E} \left[\norm{ Q^{\text{LP}}(\bx) - \bx }^2  \right] \leq \text{min} \left( \frac{p}{s^2} , \frac{\sqrt{p}}{s} \right) \norm{\bx}^2.
\end{equation}
The bound in \eqref{eq:low_prec_var} illustrates the trade-off between communication cost and quantization variance. Choosing a large $s$ reduces the variance of quantization at the cost of increasing the levels of quantization and therefore increasing the communication cost. 
\end{example}

The following example provides another quantizer which satisfies the conditions in Assumption~\ref{assump5}.

\begin{example}[Gradient sparsifier]
The gradient sparsifier denoted by $ Q^{\text{GS}}: \mathbb{R}^p \rightarrow \mathbb{R}^p$ is defined as 
\begin{equation}
    Q^{\text{GS}}_i (\bx) =
\left\{
	\begin{array}{ll}
		x_i / q_i  & \mbox{w.p. } q_i, \\
		0 & \mbox{otherwise},
	\end{array}
\right.
\end{equation}
where $q_i$ is probability that coordinate $i \in [p]$ is selected. It is easy to verify that this quantizer is unbiased, as for each $i$, $\mathbb{E} \left[Q^{\text{GS}}_i (\bx) \right] = x_i$. Moreover, one can show that the variance of this quantizer is bounded as follows,
\begin{equation}
    \mathbb{E} \left[ \norm{ Q^{\text{GS}}(\bx) - \bx }^2 \right] = \sum_{i=1}^{p} \left( \frac{1}{q_i} - 1\right) x_i^2 \leq \left( \frac{1}{q_\text{min}} - 1\right) \norm{\bx}^2,
\end{equation}
where $q_\text{min}$ denotes the minimum of probabilities $\{q_1,\cdots,q_p\}$. 
\end{example}

In the following theorem, we extend our result in Theorem~\ref{thm1} to the case that variance of quantizer may not be uniformly bounded and is proportional to the squared norm of quantizer's input.

\begin{theorem}\label{thm2}
Consider the distributed consensus optimization Problem (\ref{eq:main}) and suppose Assumptions \ref{assump1}, \ref{assump2}, \ref{assump4}, \ref{assump5} hold. Then, for each node $i$, the expected squared difference between the output of the QDGD method outlined in Algorithm \ref{alg:QDC} and the optimal solution  of Problem (\ref{eq:main}), i.e. $\widetilde{\bx}^*$ is upper bounded by
\begin{align}\label{thm_2_claim}
\bE \Big[\norm{\bx_{i,T} - \widetilde{\bx}^*}^2\Big] &\leq \mathcal{O} \Bigg( \Bigg(  \frac{4n c^2_2 D^2 \left(3+ 2L/\mu\right)^2 }{(1-\beta)^2} \nonumber\\
 &\quad   \quad   + \frac{4 c_1 {  n}\widetilde{B}^2 \eta^2 \norm{W-W_D}^2}{\mu c_2} \Bigg)\frac{1}{T^{\delta}} \Bigg),
\end{align}
for $\eps=\frac{c_1}{T^{3\delta/2}}$, $\al = \frac{c_2}{T^{\delta/2}}$, any $\delta \in (0,1/2)$ and $T \geq \widetilde{T}_0$, where $c_1$, $c_2$ and $\widetilde{T}_0$ are positive constants independent of $T$, and
\begin{align}
    \widetilde{B}^2 = \frac{4 c^2_2 D^2 \left(3+ 2L/\mu\right)^2 }{(1-\beta)^2} + \frac{4(f_0 - f^*)}{\mu}. 
\end{align}
\end{theorem}

\begin{proof}
See Appendix \ref{app_proof_of_theorem_2}.
\end{proof}

The result in Theorem~\ref{thm2} shows that under Assumption~\ref{assump5}, the proposed QDGD method converges to the optimal solution at a sublinear rate of $ \cO \left( {T^{-\delta}} \right)$ which matches the result in Theorem~\ref{thm1}. However, the lower bound on the total number of iterations $ \widetilde{T}_0$ for the result in Theorem~\ref{thm2} is in general larger than  $ {T}_0$ for the result in Theorem~\ref{thm1}. The exact expression of $ \widetilde{T}_0$ could be found in Appendix \ref{app_proof_of_theorem_2}.

\section{Optimal quantization level for reducing overall communication cost}\label{sec:extension}
In this section, we aim to study the trade-off between number of iterations until achieving a target accuracy and quantization levels. Indeed, by increasing quantization levels the variance of quantization reduces and the total number of iterations to reach a specific accuracy decreases, but the communication overhead of each round is higher as we have to transmit more bits. Conversely, if we use a quantization with a small number of levels the communication cost per iteration will be low; however, the total number of iterations could be very large. The fundamental question here is how to choose the quantization levels to optimize the overall communication cost which is the product of number of iterations and communication cost of each iteration.

In this section, we only focus on unbiased quantizers for which the variance is proportionally bounded with the  squared norm of the quantizer's input vector, i.e., for any $\bx \in \mathbb{R}^p$ it holds that $\mathbb{E}\left[Q(\bx) | \bx \right] = \bx$ and $\mathbb{E} \left[\| Q(\bx) - \bx\|^2  | \bx \right] \leq \eta^2 \norm{\bx}^2$ for some fixed constant $\eta$. Theorem \ref{thm2} characterizes the (order-wise) convergence of the proposed algorithm considering this assumption. More precisely,  using the result in Theorem \ref{thm2} and \eqref{thm_2_claim} we can write for each node $i$:
\begin{align} \label{eq:minT}
    & \quad \bE \left[  \norm{\bx_{i,T} - \widetilde{\bx}^*}^2\right]  \nonumber\\ 
  & \leq \left[ \frac{4n c^2_2 D^2 \left(3+ 2L/\mu\right)^2 }{(1-\beta)^2}   + \frac{4 c_1 {  n}\widetilde{B}^2 \eta^2 \norm{W\!-\!W_D}^2}{\mu c_2}  \right] \frac{1}{T^{\delta}},
\end{align}
where the approximation is due to considering dominant terms in $B_1(T)$ and $B_2(T)$ (See Appendix \ref{app_proof_of_lemma_2} and \ref{app_proof_of_theorem_2} for notations and details of derivations). Therefore, given a target relative deviation error $\rho$ and using \eqref{eq:minT}
, the algorithm needs to iterate at least $T(\rho)$ where
\begin{align}\label{eq:T-eps}
     T(\rho) &:= \Bigg[ \frac{4n c^2_2 D^2 \left(3+ 2L/\mu\right)^2 }{(1-\beta)^2}  \nonumber\\ 
     &\quad\quad + \frac{4 c_1 {  n}\widetilde{B}^2 \eta^2 \norm{W-W_D}^2}{\mu c_2} \Bigg]^{1/\delta} \left( \frac{1}{\rho \norm{\widetilde{\bx}^*}^2} \right) ^{1/\delta}.
\end{align}

It is shown in \cite{alistarh2017qsgd} that for the low-precision quantizer defined in \eqref{eq:low_per_def} and \eqref{eq:low_per_def_2} there exists an encoding scheme $\text{Code}_s$ such that for any $\bx \in \mathbb{R}^p$ and $s^2 + \sqrt{p} \leq p/2$, the communication cost of the quantized vector satisfies 
\begin{align}\label{eq:code-length}
    &\mathbb{E} \left[| \text{Code}_s(Q^{\text{LP}}(\bx)) | \right]  \nonumber\\
    & \quad \quad   \leq b + \left( 3 + \frac{3}{2} \log^* \left( \frac{2(s^2+p)}{s^2+\sqrt{p}} \right) \right) (s^2+\sqrt{p}),
\end{align}
where $\log^*(x)=\log(x) + \log \log (x) + \cdots=(1+o(1))\log (x)$ and $b$ denotes the number bits for representing one floating point number ($b \in \{32,64\}$ are typical values). For large $s$, \cite{alistarh2017qsgd} also proposes a simple encoding scheme $\text{Code}'_s$ which is proved to impose no more than the following communication cost on the quantized vector
\begin{align}\label{eq:code-prime-length}
    &\mathbb{E} \left[| \text{Code}'_s(Q^{\text{LP}}(\bx)) | \right]  \nonumber\\
    & \quad \leq b +\left(\frac{5}{2}+ \frac{1}{2}  \log^* \left(1+ \frac{s^2+\min(d,s \sqrt{p})}{p} \right)\right) p.
\end{align}

Now we can easily derive the expected total communication cost (in bits) of a quantized decentralized consensus optimization in order for each agent to achieve a predefined target error. For instance, assume that the low-precision quantizer described above is employed for the quanization operations. Using this quantizer, the expected communication cost (in bits) for transmitting a single $p$-dimensional real vector is represented in (\ref{eq:code-length}) and (\ref{eq:code-prime-length}) for two sparsity regimes of the tuning parameter $s$. 

On the other hand, in order for each agent to obtain a relative error $\rho$, the proposed algorithm iterates $T(\rho)$ times as denoted in (\ref{eq:T-eps}). Therefore, the total (expected) communication cost across all of the $n$ agents is upper-bounded by $n T(\rho)\cdot \mathbb{E} \left[| \text{Code}_s(Q^{\text{LP}}(\bx)) | \right] $ and  $n T(\rho) \cdot \mathbb{E} \left[| \text{Code}'_s(Q^{\text{LP}}(\bx)) | \right] $ for small and large $s$, respectively.

\begin{remark}
We can derive the total communication cost for the vanilla DGD method (\cite{yuan2016convergence}), as well. DGD method updates the iterations as follows:
\begin{equation}\label{eq:DGDwQ2}
    \bx_{i,t+1} =   w_{ii} \bx_{i,t} + \sum_{ j \in  \mathcal{N}_i } w_{ij} \bx_{j,t} - \al \gr f_i(\bx_{i,t}),
\end{equation}
where $\al = c/\sqrt{T}$ is the stepsize. DGD guarantees the following convergence rate for strongly convex objectives:
\begin{align}
    \norm{\bx_{i,T} - \widetilde{\bx}^*}^2 
    &\leq
    \frac{(3 + 2L/\mu)^2 D^2}{(1-\beta)^2} \al^2 \nonumber\\
    &= 
    \frac{c^2(3 + 2L/\mu)^2 D^2}{(1-\beta)^2} \frac{1}{T}.
\end{align}
Hence, to reach the $\rho$ approximation of the global optimal, DGD requires the total number of iterations
\begin{align}
    T_{\text{DGD}}(\rho) = \frac{c^2(3 + 2L/\mu)^2 D^2}{(1-\beta)^2} \frac{1}{\rho \norm{\widetilde{\bx}^*}^2}.
\end{align}
Given that each decision vector requires $b p$ number of bits in an implementation of DGD (without quantization), the  DGD method induces the communication cost of $n T_{\text{DGD}}(\rho) b p$.
\end{remark}

In the following, we numerically evaluate the communication cost of the proposed QDGD method for the following  least squares problem 
\begin{equation} \label{eq:leasetsqr}
    \min_{\bx \in \mathbb{R}^p} f(\bx) = \sum_{i=1}^{n} \frac{1}{2} \norm{\bA_i \bx - \bb_i}^2.
\end{equation}
We assume that the network contains $n=50$ agents that collaboratively aim to solve problem (\ref{eq:leasetsqr}) over the real field of size $p=200$. The elements of the random matrices $\mathbf{A}_i\in \mathbb{R}^{p\times p}$ and the solution $\widetilde{\bx}^*$ are picked from the normal distribution $\mathcal{N}(0,1)$. Moreover, we let $\bb_i = \mathbf{A}_i \widetilde{\bx}^* +\mathcal{N}(0,0.1 I_p)$. All nodes update their local variables with respect to the proposed algorithm and send the quantized updates to the neighbors using a low-precision quantizer with $s$ quantization levels and $b=64$ bits for representing one floating point number, until they satisfy the predefined relative error $\rho = 10^{-2}$. The underlying graph is an  Erd\H{o}s-R\'enyi with edge probability $p_c=0.35$.  The edge weight matrix is picked as $W = I - \frac{2}{3 \lambda_{\text{max}}(\mathbf{L})} \mathbf{L}$ where $\mathbf{L}$ is the Laplacian with $\lambda_{\text{max}}(\mathbf{L})$ as its largest eigenvalue. We also set $\delta=0.1$.

\begin{center}
\begin{table}[t]  
  \begin{tabular}{cccc}
    \toprule
            \begin{tabular}{@{}c@{}}\# quantization  \\ levels \end{tabular}
           & 
           \begin{tabular}{@{}c@{}}\# iterations  \\ ($\times 10^2$)\end{tabular}
           &
           \begin{tabular}{@{}c@{}} code length   \\ per vector (bits)\end{tabular}
           &
           \begin{tabular}{@{}c@{}}communication cost   \\ (bits) ($\times 10^{7}$)\end{tabular} \\ \midrule
    $s=1$    & $10800$   & $216.9$   & $1171$  \\
    $s=50$   & $11.6$   & $949.8$   & $5.5$  \\
    $s^*=77$ & $9.91$   & $1062$   & $5.27$   \\
    $s=10^3$   & $8.79$  & $1793$  & $7.88$ \\
    $s=10^5$   & $8.78$  & $3122$  & $13.71$ \\
    $s=10^{10}$   & $8.78$  & $6443$  & $28.3$ \\
    $s=10^{15}$   & $8.78$  & $9765$  & $42.9$ \\
    $s=10^{19}$   & $8.78$  & $12420$  & $54.56$ \\ \bottomrule
  \end{tabular}
  \vspace{.2 cm}
  \caption{Quantization-communication trade-off for least squares problem}\label{tab:trade-off}
    \vspace{-.6 cm}
\end{table}
\end{center}

Table \ref{tab:trade-off} represents the total expected communication cost (in bits, as computed using (\ref{eq:T-eps}), (\ref{eq:code-length}) and (\ref{eq:code-prime-length}))   induced by the proposed algorithm to solve (\ref{eq:leasetsqr}) using the low-precision quantizer --as described above-- for four representative cases. As observed from this table and expected from the theoretical derivations, larger number of quantization levels translates to less noisy quantization and hence fewer iterations. Also, larger number of quantization levels induces more communication cost for each transmitted quantized data variable which results in larger code length per vector. However, the average total communication cost does not necessarily follow a monotonic trend. As Table \ref{tab:trade-off} shows, the optimal $s^*=77$ induces the smallest total communication cost among all levels $s \geq 1$. Moreover, Table \ref{tab:trade-off} demonstrates the significant gain of picking the optimal levels $s^*$ compared to the larger ones.

  \vspace{-.5 cm}
\section{Numerical Experiments}\label{sec:simulation}

In this section, we evaluate the performance of the proposed QDGD Algorithm on decentralized quadratic minimization and ridge regression problems and demonstrate the effect of various parameters on the relative expected error rate.  We carry out the simulations on artificial and real data sets corresponding to quadratic minimization  and ridge regression problems, respectively. In both cases, the graph of agents is a connected Erd\H{o}s-R\'enyi with edge probability $p_c$. We set the edge weight matrix to be $W = I - \frac{2}{3 \lambda_{\text{max}}(\mathbf{L})} \mathbf{L}$ where $\mathbf{L}$ is the Laplacian with $\lambda_{\text{max}}(\mathbf{L})$ as its largest eigenvalue.

\subsection{Decentralized quadratic minimization}
In this section, we evaluate the performance of the proposed QDGD Algorithm on minimizing a distributed quadratic objective. We pictorially demonstrate the effect of quantization noise and graph topology on the relative expected error rate. 

Consider the quadratic optimization problem 
\begin{equation} \label{eq:quadratic}
    \min_{ \bx \in \mathbb{R}^p} \ f(\bx) = \sum_{i=1}^{n} \frac{1}{2} \bx^\top \bA_i \bx +  \bb_i^\top \bx,
\end{equation}
where $f_i(\bx) = \frac{1}{2} \bx^\top \bA_i \bx +  \bb_i^\top \bx$ denotes the local objective function of node $i\in [n]$. The unique solution to (\ref{eq:quadratic}) is therefore $\widetilde{\bx}^* = -\left( \sum_{i=1}^{n} \bA_i\right)^{-1}\left(\sum_{i=1}^{n}\bb_i\right)$. We pick diagonal matrices $\bA_i$ such that $p/2$ of the diagonal entries of each $\bA_i$ are drawn from the set $\{1,2,2^2\}$ and the other $p/2$ diagonal entries are drawn from the set $\{1,2^{-1},2^{-2}\}$, all uniformly at random. Entries of vectors $\bb_i$ are randomly picked from the interval $(0,1)$. In our simulations, we let an additive noise model the quantization error, i.e. $Q(\bx)=\bx+ {\eta}$ where $\eta \sim \mathcal{N}(0,\frac{\sigma^2}{p} I_p)$.

We first consider  a connected Erd\H{o}s-R\'enyi graph of $n=50$ nodes and connectivity probability of $p_c=0.35$ and dimension $p=20$. Fig.~\ref{sigma_plot} shows the convergence rate corresponding to three values of quantization noise $\sigma^2 \in \{2,20,200\}$ and $\delta=3/8$, compared to {{the theoretical upper bound derived in Theorem \ref{thm1} in the logarithmic scale.}} For each plot, stepsizes are  pick as $\eps={c_1}/{T^{3\delta/2}}$ and $\al = {c_2}/{T^{\delta/2}}$ where the constants $c_1,c_2$ are finely tuned. As expected, Fig.~\ref{sigma_plot} shows that the error rate linearly scales with the quantization noise; however, it does not saturate around a non-vanishing residual, regardless the variance. Moreover, Fig.~\ref{sigma_plot} demonstrates that the convergence rate closely follows the upper bound derived in  Theorem \ref{thm1}. For instance, for the plot corresponding to $\sigma^2=200$, the relative errors are evaluated as ${e_{T_1}}/{e_0}=0.1108$ and  ${e_{T_2}}/{e_0}=0.0634$ for $T_1=800$ and $T_2=3200$, respectively. Therefore, ${e_{T_2}}/{e_{T_1}} \approx 0.57$ which is upper bounded by $(\frac{T_1}{T_2})^{\delta} \approx 0.59$.

To observe the effect of graph topology, quantization noise variance is fixed to $\sigma^2=200$ and we varied the connectivity ratio by picking three different values, i.e. $p_c \in \{0.35,0.5,1\}$ where $p_c=1$ corresponds to the complete graph case.  We also fix the parameter $\delta=3/8$ and accordingly  pick the stepsizes $\eps={c_1}/{T^{3\delta/2}}$ and $\al = {c_2}/{T^{\delta/2}}$ where the constants $c_1,c_2$ are finely tuned. As Fig. \ref{graph_plot} depicts, for the same number of iterations, deviation from the optimal solution tends to increase as the graph is gets sparse. In other words, even noisy information of the neighbor nodes improves the gradient estimate for local nodes. It also highlights the fact that regardless of the sparsity of the graph, the proposed QDGD algorithm guarantees the consensus to the optimal solution for each local node, as long as the graph is connected.\\

\begin{figure}[h]
\vspace{-.5cm}
\includegraphics[width=.5\textwidth ]{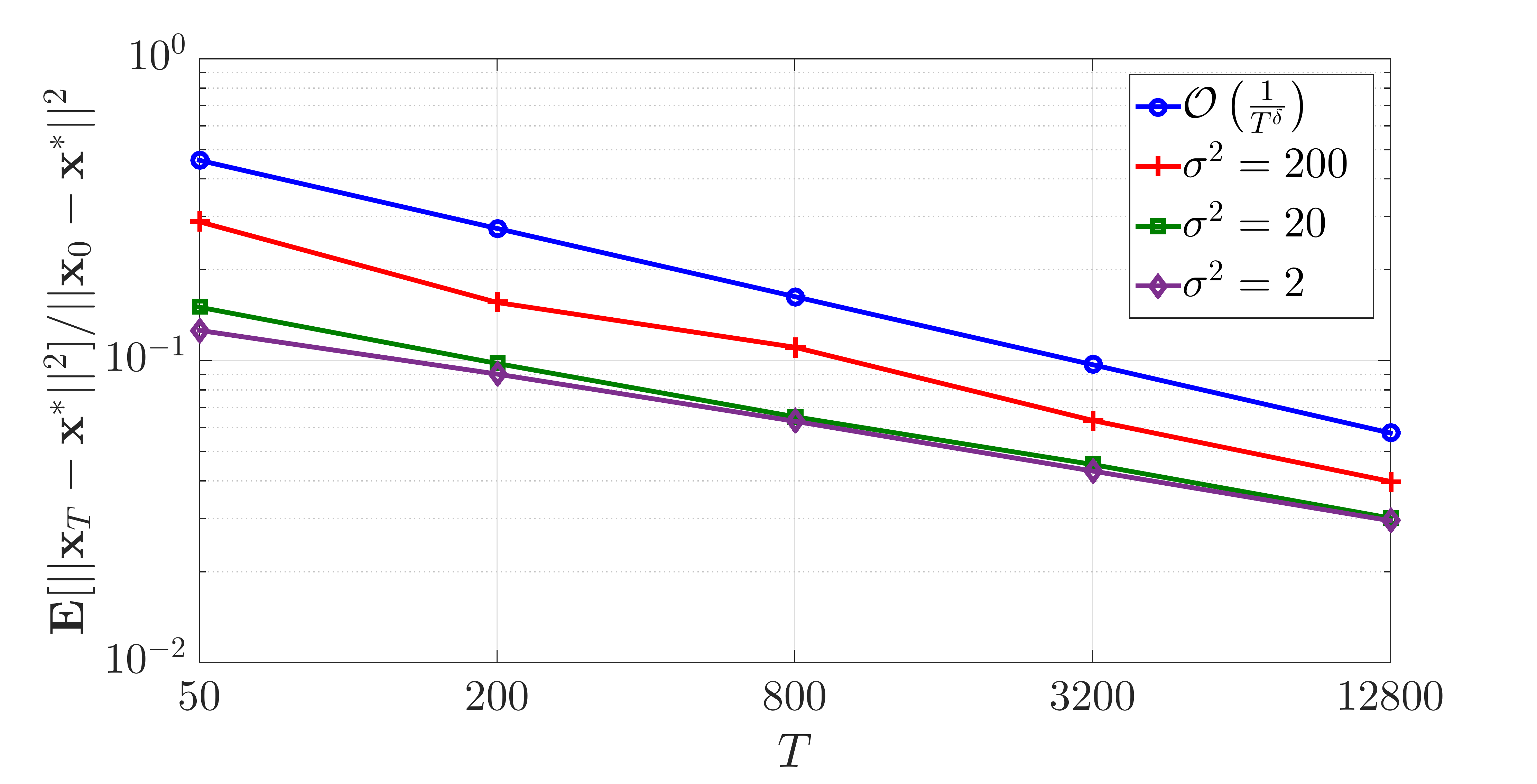}
\centering
\vspace{-.5cm}
\caption{Relative optimal squared error for three values of quantization noise variance: $\sigma^2 \in \{2,20,200\}$, compared with the order of upper bound.}
\label{sigma_plot}
\end{figure}

\begin{figure}[h]
\vspace{-.5cm}
\includegraphics[width=.5\textwidth ]{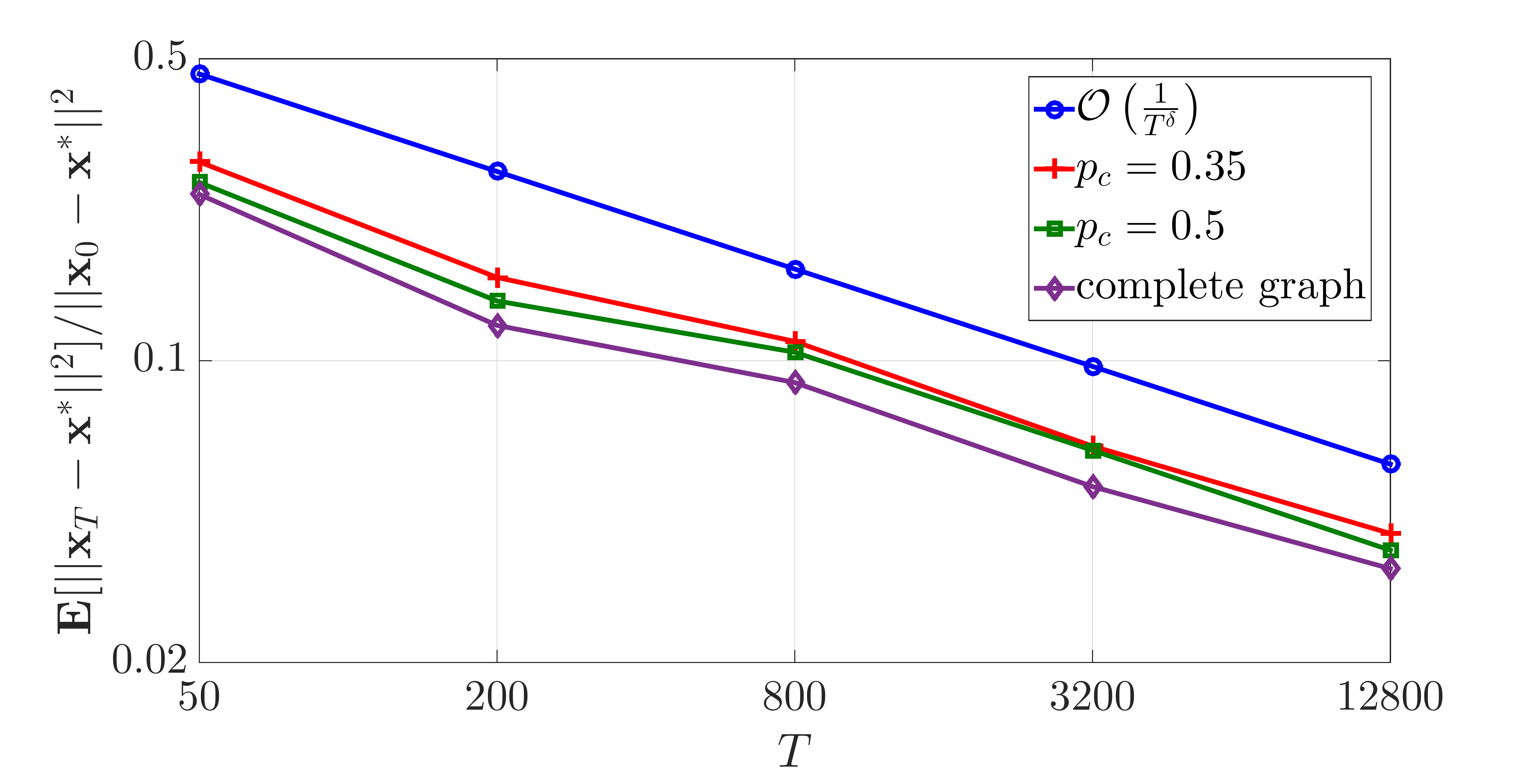}
\centering
\vspace{-.5cm}
\caption{Relative optimal squared error for three vales of graph connectivity ratio: $p_c \in \{0.35,0.5,1\}$, compared with the order of upper bound.}
\label{graph_plot}
\end{figure}

\subsection{Decentralized ridge regression}

Consider the ridge regression problem:
\begin{equation} \label{eq:ridgereg}
    \min_{\bx \in \mathbb{R}^p} f(\bx) = \sum_{j=1}^{D} \norm{\mathbf{a}_j \bx - b_j}^2 + \frac{\lambda}{2} \norm{\bx}_2^2,
\end{equation}
over the data set $\mathcal{D} = \{ (\mathbf{a}_j , b_j): j=1,\cdots,D\}$ where each pair $(\mathbf{a}_j , b_j)$ denotes the predictors-response variables corresponding to data point $j\in[D]$ where $\mathbf{a}_j \in \mathbb{R}^{1 \times p}, b_j \in \mathbb{R}$ and $\lambda > 0$ is the regularization parameter. To make this problem decentralized, we pick $n$ agents and uniformly divide the data set $\mathcal{D}$ among the $n$ agents, i.e., each agent is assigned with $d = D/n$ data points. Therefore, (\ref{eq:ridgereg}) can be decomposed as follows:
\begin{equation} \label{eq:ridgereg2}
    \min_{\bx \in \mathbb{R}^p} f(\bx) = \sum_{i=1}^{n} f_i(\bx),
\end{equation}
where the local function corresponding to agent $i \in [n]$ is
\begin{equation} \label{eq:ridgereg3}
    f_i(\bx) = \norm{\bA_i \bx - \bb_i}^2 + \frac{\lambda}{2n} \norm{\bx}^2,
\end{equation}
and 
\begin{align} \label{eq:ridgereg3}
    \bA_i &= [\mathbf{a}_{(i-1)d+1};\cdots;\mathbf{a}_{id}] \in \mathbb{R}^{d \times p},\\
    \bb_i &= [b_{(i-1)d+1};\cdots;b_{id}] \in \mathbb{R}^{d}.
\end{align}
The unique solution to (\ref{eq:ridgereg2}) is
\begin{equation} \label{eq:ridgereg4}
    \widetilde{\bx}^* = \left(\sum_{i=1}^{n} \bA^{\top}_i\bA_i + \lambda {I}\right)^{-1}\left(\sum_{i=1}^{n} \bA^{\top}_i\bb_i\right).
\end{equation}
To simulate  the decentralized ridge regression (\ref{eq:ridgereg2}), we pick ``Pen-Based Recognition of Handwritten Digits Data Set" \cite{Dua:2017} and use $D=5000$ training samples with $p=16$ features and $10$ possible labels corresponding to digits $\{\textrm`0\textrm', \textrm`1\textrm', \cdots, \textrm`9\textrm'\}$. We pick $\lambda=2$ and consider a connected Erd\H{o}s-R\'enyi graph with $n=50$ agents and edge probability $p_c$, i.e. each assigned with $d=100$ data points.  The decision variables are quantized according to the low-precision quantizer with quantization level $s$, as described in Example \ref{ex:lp}.

Firstly, we fix $p_c=0.25$ and $s=1$ and vary the tuning parameter $\delta$. Fig. \ref{sigma_plot_ridge} depicts the convergence trend corresponding to two values $\delta \in \{0.175,0.275\}$. For each pick of $\delta$, the stepsizes are set to $\eps={c_1}/{T^{3\delta/2}}$ and $\al = {c_2}/{T^{\delta/2}}$ with finely tuned constants $c_1,c_2$.

Secondly, to observe the effect of graph density, we let the quantization level be $s=1$ and vary the graph configuration. For $\delta=0.275$, Fig. \ref{graph_plot_ridge} shows the resulting convergence rates for Erd\H{o}s-R\'enyi random graphs with two vales of graph connectivity ratio $p_c \in \{0.25,0.45\}$, complete graph and cycle graph.

\begin{figure}[t]
\includegraphics[width=.5\textwidth ]{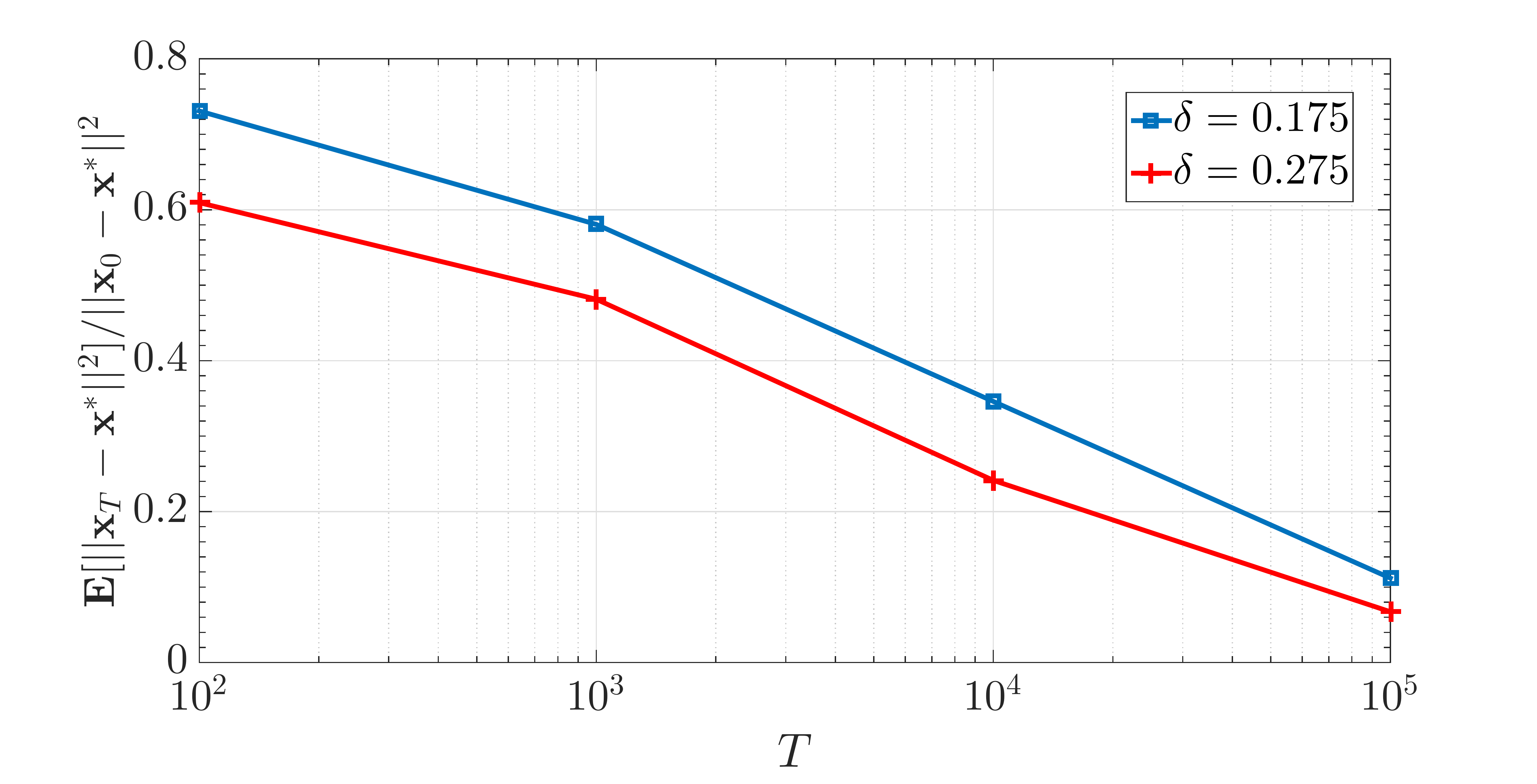}
\centering
\vspace{-.5cm}
\caption{Relative optimal squared error for two vales of $\delta$: $\delta \in \{0.175,0.275\}$.}
\label{sigma_plot_ridge}
\end{figure}

\begin{figure}[t]
\includegraphics[width=.5\textwidth ]{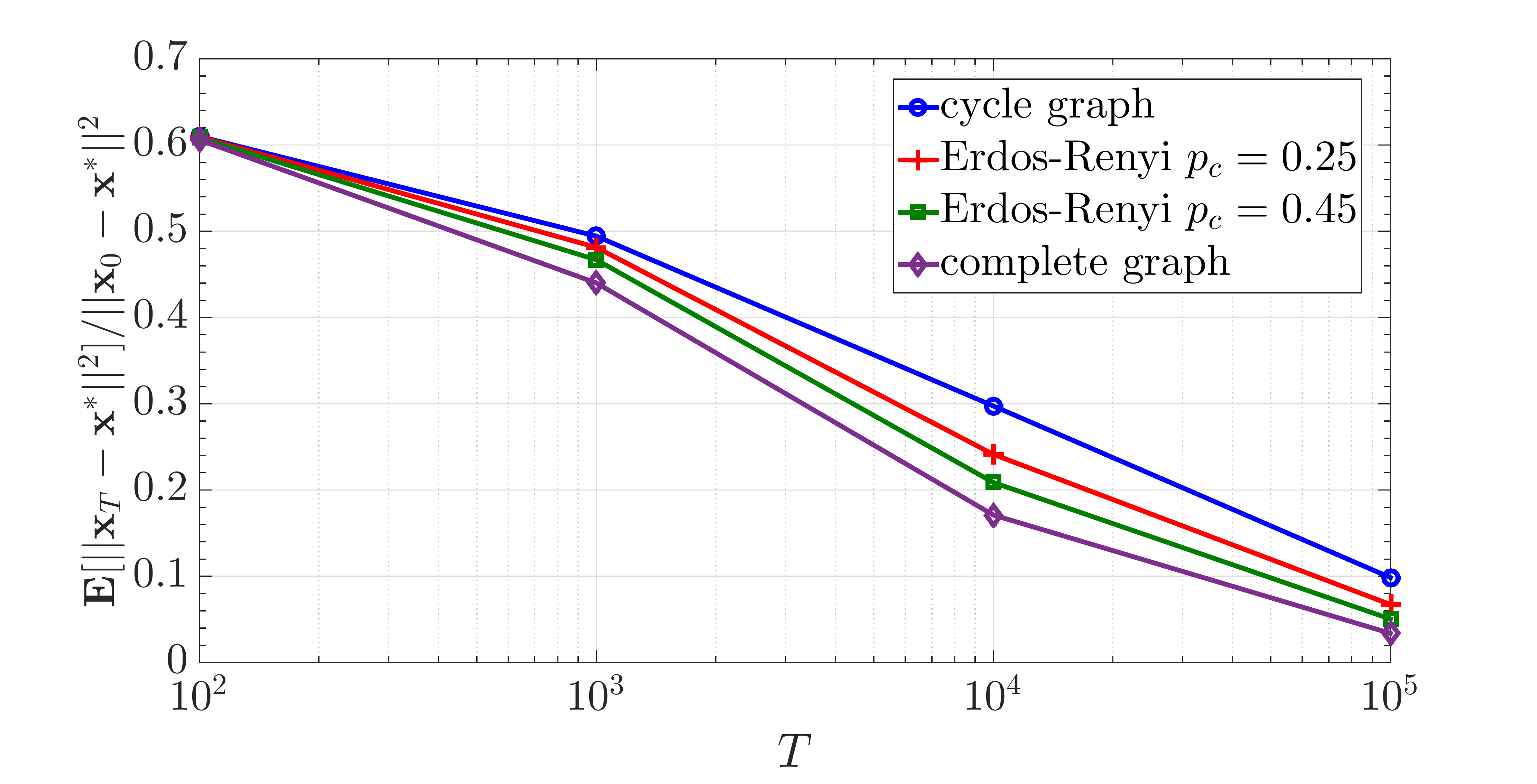}
\centering
\vspace{-.5cm}
\caption{Relative optimal squared error for Erd\H{o}s-R\'enyi random graphs with two vales of graph connectivity ratio: $p_c \in \{0.25,0.45\}$, complete graph and cycle graph.}
\label{graph_plot_ridge}
\end{figure}

\subsection{Logistic regression}
To further evaluate the proposed method with other benchmarks, in this section we consider the logistic regression where the goal is to learn a classifier $\bx$ to predict the labels $b_{j} \in \{+1,-1\}$. More specifically, consider the regularized logistic regression problem as follows:
\begin{equation} \label{eq:logisticreg}
    \min_{\bx \in \mathbb{R}^p} f(\bx) = 
    \frac{1}{n} \sum_{j=1}^{D} \log \left( 1 + \exp \left(- b_j \mathbf{a}_{j} \bx \right) \right)
    +
    \frac{\lambda}{2} \norm{\bx}_2^2,
\end{equation}
where $b_j \in \{+1,-1\}$ denotes the label of the $j$th data-point corresponding to the feature vector $\mathbf{a}_{j} \in \mathbb{R}^{1 \times p}$. The total $D$ data-points are distributed among the $n$ nodes such that each node is assigned with $d=D/n$ samples. The underlying network is an Erd\"{o}s-R\'enyi graph with $n=50$ nodes and connectivity probability $p_c=0.45$. We generate a data-set of $D=5000$ samples as follows. Each sample with label $+1$ is associated with a feature vector of $p=4$ random gaussian entries with mean $\mu$ and variance $\gamma^2$. Similarly, samples with labels $-1$ are associated with a feature vector of random gaussian entries with mean $-\mu$ and variance $\gamma^2$. We let $\mu=3$ and $\gamma^2=1$.

In the implementation of the QDGD method, we pick the parameter $\delta=0.45$ and accordingly  pick the stepsizes $\eps={c_1}/{T^{3\delta/2}}$ and $\al = {c_2}/{T^{\delta/2}}$ where the constants $c_1,c_2$ are finely tuned.

As a benchmark, we compare our proposed QDGD method with the naive DGD algorithm \cite{yuan2016convergence} in which we let the nodes exchange quantized decision variables. That is, the update rule at node $i$ and iteration $t$ in this benchmark is 
\begin{equation}\label{eq:DGDwQ}
    \bx_{i,t+1} =   w_{ii} \bx_{i,t} + \sum_{ j \in  \mathcal{N}_i } w_{ij} \bz_{j,t} - \al \gr f_i(\bx_{i,t}),
\end{equation}
where we pick the stepsize $\al = c/T$ with finely tuned constant $c$.

In both methods, we use the low-precision quantizer in \eqref{eq:low_per_def} with $s$ levels of quantization. Note that unlike the proposed QDGD, the update rule in this benchmark employs only one step-size $\al$. In addition to this comparison, we illustrate the effect of the quantization level $s$ on the convergence of the two methods. Fig. \ref{logistic_plot} demonstrates the loss values resulting from the two methods for five picks of $T \in \{750,1000,1250,1500,1750\}$. As we mentioned earlier, the proposed QDGD is an exact method, i.e. the local models converge to the global optimal model with any desired optimality gap. However, a naive generalization of the existing methods (e.g. DGD) with quantization (e.g. in \eqref{eq:DGDwQ}) will result in a convergence to a neighborhood of the global optimal.

Fig. \ref{logistic_plot} also shows that for less noisy quantizers (larger $s$), nodes receive more accurate models from the neighbors and hence they achieve a smaller loss within a fixed number of iterations.

\begin{figure}[t]
\includegraphics[width=.5\textwidth ]{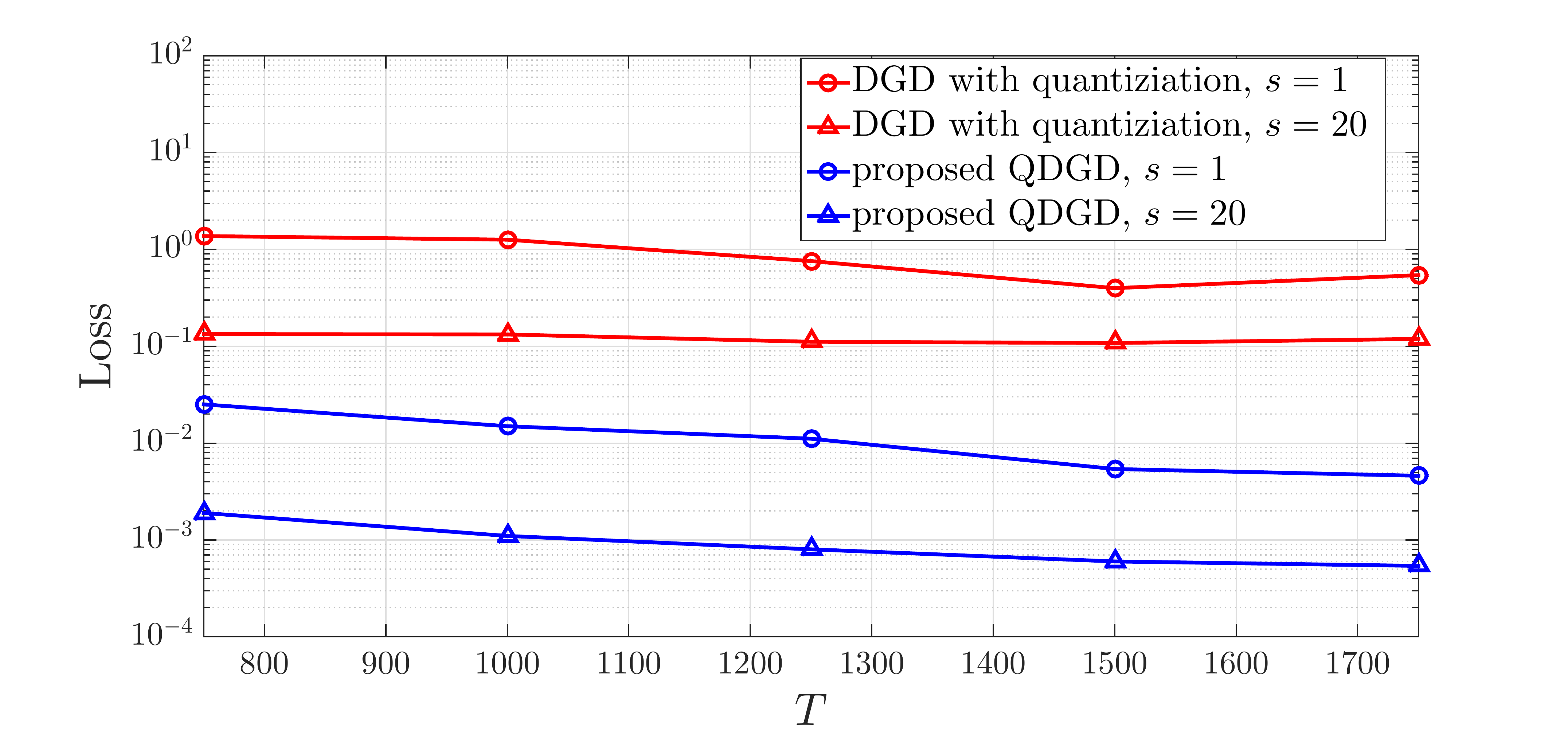}
\centering
\caption{Comparing the proposed QDGD method and the naive DGD with quantization (see \eqref{eq:DGDwQ}); and varying the quantizations levels $s \in \{1,20\}$.}
\label{logistic_plot}
\end{figure}

\section{Conclusion}\label{sec:conclusion}
We proposed the QDGD algorithm to tackle the problem of quantized decentralized consensus optimization. The algorithm updates the local decision variables by combining
the quantized messages received from the neighbors and the
local information such that proper averaging is performed
over the local decision variable and the neighbors' quantized
vectors. Under customary conditions for quantizers, we proved that the QDGD algorithm achieves a vanishing consensus error in mean-squared sense, and verified our theoretical results with numerical studies. Following our preliminary work \cite{reisizadeh2018quantized}, there has been a growing interest in developing quantized decentralized optimization methods \cite{doan2018accelerating,lee2018finite,zhang2018compressed}. In particular, in \cite{doan2018accelerating} authors propose to use \emph{adaptive} quantization which is kept tuned during the convergence. Authors in \cite{zhang2018compressed} relax the convexity assumption and develop another quantized method for a more general class of objective functions.

An interesting future direction is to establish a fundamental trade-off between the convergence rate of quantized consensus algorithms and the communication. More precisely, given a target convergence rate, what is the minimum number of bits that one should communicate in decentralized consensus? Another interesting line of research is to develop novel source coding (quantization) schemes that have low computation complexity and are information theoretically near-optimal in the sense that they have small communication load and fast convergence rate. Lastly, developing such communication-efficient decentralized optimization methods for convex or non-convex functions are highly critical given the rise of deep neural networks in the learning literature, which is another line in our future directions.


\begin{appendices}

\section{Proof of Lemma \ref{lemma1}}\label{app_proof_of_lemma_1}

To prove the claim in Lemma \ref{lemma1} we first prove the following intermediate lemma.

 \begin{lemma} \label{lemma3}
Consider the non-negative sequence $e_t$ satisfying the inequality 
\begin{equation}\label{eq:recursive_expression}
    e_{t+1}\leq \left(1-\frac{a}{T^{2\delta}}\right) e_t +\frac{b}{T^{3\delta}},
\end{equation}
for $t\geq 0$, where $a$ and $b$ are positive constants, $\delta\in[0,1/2)$, and $T$ is the total number of iterations. Then, after $T \geq \max \left\{ a^{1/(2\delta)}, \exp\left(\exp\left(1/\left(1-2\delta\right)\right)\right) \right\}$ iterations the iterate $e_T$ satisfies 
\begin{equation}\label{lkjnjlknl}
    e_T \leq \mathcal{O}\left(\frac{b}{aT^{\delta}}\right).
\end{equation}
\end{lemma}

\begin{proof}
Use the expression in \eqref{eq:recursive_expression} for steps $t-1$ and $t$ to obtain
\begin{align}
    e_{t+1} 
    & \leq  \left(1-\frac{a}{T^{2\delta}}\right)^2 e_{t-1} \nonumber\\ 
    & \quad +\left[1+\left(1-\frac{a}{T^{2\delta}}\right) \right]\frac{b}{T^{3\delta}},
\end{align}
where $T \geq a^{1/(2\delta)}$. By recursively applying these inequalities for all steps $t\geq 0$ we obtain that  
\begin{align}
    & e_{t} \leq  \left(1-\frac{a}{T^{2\delta}}\right)^t e_{0} 
     \nonumber\\
     &+\frac{b}{T^{3\delta}} \left[1+\left(1-\frac{a}{T^{2\delta}}\right) +\dots +\left(1-\frac{a}{T^{2\delta}}\right)^{t-1} \right] \nonumber\\
     &\leq  \left(1-\frac{a}{T^{2\delta}}\right)^t e_{0} 
     +\frac{b}{T^{3\delta}} \left[\sum_{s=0}^{t-1}\left(1-\frac{a}{T^{2\delta}}\right)^{s} \right] \nonumber\\
     &\leq  \left(1-\frac{a}{T^{2\delta}}\right)^t e_{0} 
     +\frac{b}{T^{3\delta}} \left[\sum_{s=0}^{\infty}\left(1-\frac{a}{T^{2\delta}}\right)^{s} \right] \nonumber\\
     &=  \left(1-\frac{a}{T^{2\delta}}\right)^t e_{0} 
     +\frac{b}{T^{3\delta}} \left[\frac{1}{1- \left( 1-\frac{a}{T^{2\delta}} \right) } \right] \nonumber\\
     &=  \left(1-\frac{a}{T^{2\delta}}\right)^t e_{0} 
     +\frac{b}{aT^{\delta}}.  
\end{align}
Therefore, for the iterate corresponding to step $t=T$ we can write 
\begin{align}
    e_{T} 
     &\leq 
      \left(1-\frac{a}{T^{2\delta}}\right)^T e_{0} 
     +\frac{b}{aT^{\delta}}  \nonumber\\
     &\leq 
      \exp \left({-aT^{(1-2\delta)}} \right) e_{0} 
     +\frac{b}{aT^{\delta}}  \label{eq:withexp} \\
     &=
      \mathcal{O}\left(\frac{b}{ aT^{\delta}}\right),
\end{align}
and the claim in \eqref{lkjnjlknl} follows. Note that for the last inequality we assumed that the exponential term in is negligible comparing to the sublinear term. It can be verified for instance if $1-2\delta$ is of $\mathcal{O} \left( 1/\log(\log(T)) \right)$ or greater than that, it satisfies this condition. Moreover, setting $\delta=1/2$ results in a constant (and hence non-vanishing) term in (\ref{eq:withexp}). 
\end{proof}

 Now we are at the right position to prove Lemma \ref{lemma1}. We start by evaluating the gradient function of $h_{\al}$ at the concatenation of local variables at time $t \geq 1$, that is $\gr h_{\al}(\bx_t)= \big(\bI - \bW \big) \bx_t + \al \gr F(\bx_t)$. Consider the vector $\bz_t=[\bz_{1,t};\dots;\bz_{n,t}]$ as the concatenation of the quantized variant of the local updates $\bx_t=[\bx_{1,t};\dots;\bx_{n,t}]$. Then, we obtain that the expression on the right hand side of \eqref{eq:rule-matrix}, i.e.,
 \begin{equation}\label{eq:htilde}
     \widetilde{\gr}h_{\al}(\bx_t)=\big(\bW_{D} - \bW\big) \bz_{t} + \big(\bI - \bW_D \big) \bx_t + \al \gr F(\bx_{t}), 
 \end{equation}
 defines a stochastic estimate of the true gradient of $h_{\al}$ at time $t$, i.e., $\gr h_{\al}(\bx_t)$. We let $\cF^t$ denote a sigma algebra that measures the history of the system up until time $t$ and take the conditional expectation $\Expc[\cdot|\cF^t]$ from both sides of (\ref{eq:htilde}). It yields
 \begin{align}
     &\Expc \left[  \widetilde{\gr}h_{\al}(\bx_t) \vert \cF^t \right] \nonumber\\
     & \quad = \left(\bW_{D} - \bW \right) \Expc \left[\bz_{t}|\cF^t \right] + \left(\bI - \bW_D \right) \bx_t + \al \gr F(\bx_{t}), \nonumber\\
     & \quad = \left(\bI - \bW \right) \bx_t + \al \gr F(\bx_t)\nonumber\\
     & \quad = \gr h_{\al}(\bx_t),
 \end{align}
 where we used the fact that $\Expc \left[\bz_{t}|\cF^t \right]=\bx_t$ (Assumption \ref{assump3}). Hence, $\widetilde{\gr}h_{\al}$ is an unbiased estimator for the true gradient $\gr h_{\al}$. Now, we can rewrite the update rule (\ref{eq:rule-matrix}) as 
 \begin{equation} \label{eq:hsgd}
     \bx_{t+1} = \bx_{t} - \eps \widetilde{\gr}h_{\al}(\bx_t),
 \end{equation}
 which resembles the stochastic gradient descent (SGD) update with step-size $\eps$ for minimizing the objective function $h_{\al}(\bx)$ over $\bx \in \mathbb{R}^{np}$. Intuitively, one can expect that, for proper pick of step-size, the the sequence $\{\bx_t;t=1,2,\dots\}$ produced by update rule (\ref{eq:hsgd}) converges to the unique minimizer of $h_{\al}(\bx)$. More precisely, we can write for $t\geq 1$,
 \begin{align}
& \bE \left[ \norm{\bx_{t+1} - \bx^*_{\al}}^2 \vert \cF^t\right] \nonumber\\
&\quad = \bE \left[\norm{\bx_{t} - \eps \widetilde{\gr}h_{\al}(\bx_t) - \bx^*_{\al}}^2 | \cF^t\right]  \nonumber\\
 &\quad= \norm{\bx_{t} - \bx^*_{\al}}^2 -2 \eps \left\langle \bx_{t} - \bx^*_{\al}, \bE \left[\widetilde{\gr}h_{\al}(\bx_t)| \cF^t \right] \right\rangle \nonumber\\
 &\quad \quad + \eps^2 \bE \left[\norm{\widetilde{\gr}h_{\al}(\bx_t)}^2 | \cF^t\right]  \nonumber\\
 &\quad= \norm{\bx_{t} - \bx^*_{\al}}^2 -2 \eps \left\langle \bx_{t} - \bx^*_{\al}, {\gr}h_{\al}(\bx_t) \right\rangle \nonumber\\
 &\quad \quad + \eps^2 \bE \left[ \norm{\widetilde{\gr}h_{\al}(\bx_t)}^2 | \cF^t\right]  \nonumber\\
 &\quad\leq \left(1-2 \mu_{\al} \eps \right)\norm{\bx_{t} - \bx^*_{\al}}^2 + \eps^2 \bE \left[\norm{\widetilde{\gr}h_{\al}(\bx_t)}^2 | \cF^t\right]. \label{eq:bound1}
\end{align}
 We have used the facts that $\widetilde{\gr}h_{\al}$ is unbiased and $h_{\al}$ is strongly convex with parameter $\mu_{\al}$. Next, we bound the second term in (\ref{eq:bound1}), that is
\begin{align}
\bE & \left[ \norm{\widetilde{\gr}h_{\al}(\bx_t)}^2 \vert \cF^t\right] \nonumber\\
&= \bE \left[ \norm{ \left(\bW_{D} - \bW\right) \bz_{t} + \left(\bI - \bW_D \right) \bx_t + \al \gr F(\bx_{t})}^2 \vert \cF^t\right] \nonumber\\
&\leq \norm{{\gr}h_{\al}(\bx_t)}^2 + \bE \left[ \norm{ \left(\bW_{D} - \bW\right)( \bz_{t} -\bx_t)}^2  \vert \cF^t\right]  \nonumber\\
&\leq L^2_{\al} \norm{\bx_{t} - \bx^*_{\al}}^2 + { n}\sigma^2 \norm{W-W_D}^2, \label{eq:bound2}
\end{align}
where we used the smoothness of $h_{\al}$ and boundedness of quantization noise. Plugging (\ref{eq:bound2}) into (\ref{eq:bound1}) yields
\begin{align}
 \bE \left[\norm{\bx_{t+1} - \bx^*_{\al}}^2 | \cF^t\right]
 &\leq \left(1-2 \mu_{\al} \eps + \eps^2 L^2_{\al} \right)\norm{\bx_{t} - \bx^*_{\al}}^2 \nonumber\\
 & \quad + \eps^2 { n}\sigma^2 \norm{W-W_D}^2. \label{eq:bound3}
\end{align}
Let us define the sequence $e_t := \bE \left[\norm{\bx_{t} - \bx^*_{\al}}^2 \right]$ as the expected squared deviation of the local variables from the optimal solution ${\bx_{\al}^*}$ at time $t \geq 1$. By taking the expectation of both sides of (\ref{eq:bound3}) with respect to all sources of randomness from $t=0$ we obtain that 
\begin{align}\label{eq:bound4}
 e_{t+1} &\leq  \big(1-2 \mu_{\al} \eps + \eps^2 L^2_{\al} \big)e_t + \eps^2 { n}\sigma^2 \norm{W-W_D}^2 \nonumber\\
 &= \big(1-\eps (2 \mu_{\al}  - \eps L^2_{\al}) \big)e_t + \eps^2 {  n}\sigma^2 \norm{W-W_D}^2.
\end{align}
Notice that for the specified choice of $\eps$ and $T\geq T_1$, we have $T^{\delta} \geq T_1^{\delta} \geq  \frac{c_1 (1+c_2 L)^2}{c_2 \mu} $ and therefore
\begin{align}
\eps &= \frac{c_1}{T^{3\delta/2}}\nonumber\\
&\leq \frac{c_2 \mu }{(1+c_2 L)^2} \cdot \frac{1}{T^{\delta/2}}\nonumber\\
&\leq \frac{\mu_{\al}}{ \big(1-\lambda_n(W) + \al L \big)^2 }\nonumber\\
&\leq \frac{\mu_{\al}}{L^2_{\al}}.
\end{align}
Therefore, (\ref{eq:bound4}) can be written as
\begin{align}\label{eq:bound5}
 e_{t+1} &\leq  \left(1-\eps \left(2 \mu_{\al}  - \eps L^2_{\al} \right) \right)e_t + \eps^2 {  n}\sigma^2 \norm{W-W_D}^2 \nonumber\\
 &\leq \left(1- \mu_{\al} \eps \right)e_t + \eps^2 {  n}\sigma^2\norm{W-W_D}^2 \nonumber\\
 &= \left(1- \frac{c_1 c_2 \mu}{T^{2\delta}} \right) e_t + \frac{c^2_1 {  n}\sigma^2 \norm{W-W_D}^2}{T^{3\delta}} .
\end{align}
Now we let $a=c_1 c_2 \mu$ and $b=c^2_1 {  n}\sigma^2 \norm{W-W_D}^2$ and employ Lemma \ref{lemma3} to conclude that 
 \begin{align}\label{eq:lemma1}
e_T &= \Expc \left[  \norm{\bx_T - \bx^*_{\al}}^2 \right]  \nonumber\\
 &\leq \mathcal{O}\left(\frac{b}{a T^{\delta}}\right) \nonumber\\
 & = \mathcal{O} \left( \frac{c_1 { n} \sigma^2 \norm{W\!-\!W_D}^2}{\mu c_2} \frac{1}{T^{\delta}} \right),
\end{align}
and the proof of Lemma \ref{lemma1} is complete.


\section{Proof of Lemma \ref{lemma2}}\label{app_proof_of_lemma_2}

First, recall the penalty function minimization in (\ref{eq:hmin}). Following sequence is the update rule associated with this problem when the gradient descent method is applied to the objective function $h_{\al}$ with the unit step-size $\gamma=1$,
\begin{equation}\label{eq:gd}
\bu_{t+1} = \bu_{t} - \gamma \gr h_{\al} (\bu_t)= \bW \bu_t - \al \gr F(\bu_t).
\end{equation}
From analysis of GD for strongly convex objectives, the sequence $\{ \bu_{t} : t=0,1,\cdots\}$ defined above exponentially converges to the minimizer of $h_{\al}$, $\bx^*_{\al}$, provided that $1=\gamma \leq \frac{2}{L_{\al}}$. The latter condition is satisfied if we make $\al \leq \frac{1+\lambda_n(W)}{L}$, implying $L_{\al} = 1-\lambda_n(W) + \al L \leq 2$. Therefore, 
\begin{align}\label{eq:gd}
 \norm{\bu_{t} - \bx^*_{\al}}^2 &\leq (1- \mu_{\al})^{t}\norm{\bu_{0} - \bx^*_{\al}}^2 \nonumber\\
 &= (1- \al \mu)^{t}\norm{\bu_{0} - \bx^*_{\al}}^2.
\end{align}
If we take $\bu_{0}=0$, then (\ref{eq:gd}) implies 
\begin{align}
    \norm{\bu_{T} - \bx^*_{\al}}^2 &\leq   (1- \al \mu)^{T}  \norm{\bx^*_{\al}}^2\nonumber \\
    &\leq 2  (1- \al \mu)^{T}  \left( \norm{\bx^* - \bx^*_{\al}}^2 + \norm{\bx^*}^2 \right) \nonumber\\
    &= 2  (1- \al \mu)^{T}  \left( \norm{\bx^* - \bx^*_{\al}}^2 + n \norm{\widetilde{\bx}^*}^2 \right) \label{eq:bounduT},
\end{align}
where $f_0 = f(0)$ and $f^* = \text{min}_{\bx \in \mathbb{R}^p} f(\bx) = f(\widetilde{\bx}^*)$. On the other hand, it can be shown \cite{yuan2016convergence} that if $\al \leq \text{min} \left\{ \frac{1+\lambda_n(W)}{L} , \frac{1}{\mu  + L } \right\}$, then the sequence $\{ \bu_{t} : t=0,1,\cdots\}$ defined in (\ref{eq:gd}) converges to the $\cO \left(\frac{\al}{1-\beta} \right)$-neighborhood of the optima $\bx^*$, i.e.,
\begin{equation}\label{eq:gd3}
    \norm{\bu_{t} - \bx^*} \leq \cO\left(\frac{\al}{1-\beta} \right).
\end{equation}
If we take $\al=\frac{c_2}{T^{\delta/2}}$, the condition $T \geq T_2$ implies that $\al \leq \text{min} \left\{ \frac{1+\lambda_n(W)}{L} , \frac{1}{\mu  + L } \right\}$. Therefore, (\ref{eq:gd3}) yields
\begin{equation}\label{eq:gd4}
    \norm{\bu_{T} - \bx^*} \leq \cO \left(\frac{\al}{1-\beta} \right).
\end{equation}
More precisely, we have the following (See Corollary 9 in \cite{yuan2016convergence}):
\begin{equation} \label{eq:exactb}
    \norm{\bu_{T} - \bx^*} \leq \sqrt{n} \left(  c^T_3 \norm{\widetilde{\bx}^*} + \frac{c_4}{\sqrt{1-c^2_3}} + \frac{\al D}{1-\beta} \right),
\end{equation}
where
\begin{equation}
    c^2_3 = 1 - \frac{1}{2} \cdot \frac{\mu L}{\mu + L}  \al ,
\end{equation}
\begin{align}
	\frac{c_4}{\sqrt{1-c^2_3}} &= \frac{\al L D}{1-\beta} \sqrt{ 4 \left( \frac{\mu+L}{\mu L} \right)^2 - 2 \cdot \frac{\mu+L}{\mu L} \al } \nonumber \\
	& \leq \frac{2 \al  D }{  (1-\beta)} \left(1+L / \mu \right).
\end{align}
From (\ref{eq:exactb}) and (\ref{eq:gd4}), we have for $T \geq T_2$
\begin{align} \label{eq:boundT2}
   \norm{\bx^*_{\al} - \bx^*}^2 &= \norm{\bx^*_{\al} -\bu_{T}+\bu_{T}- \bx^*}^2 \nonumber\\
   &\leq 2\norm{\bx^*_{\al} -\bu_{T}}^2 + 2\norm{\bu_{T}- \bx^*}^2 \nonumber\\
   &\leq 4  (1- \al \mu)^{T}  \left( \norm{\bx^* - \bx^*_{\al}}^2 + n \norm{\widetilde{\bx}^*}^2 \right) \nonumber\\
   & \quad + 2n \Bigg(  \left( 1 - \frac{1}{2} \cdot \frac{\mu L}{\mu + L}  \al \right)^{T/2} \norm{\widetilde{\bx}^*} \nonumber\\
   & \quad +  \frac{\al D}{1-\beta} \left(3+ 2L / \mu \right) \Bigg)^2.
\end{align}
Note that for our pick $\al = \frac{c_2}{T^{\delta/2}}$, we can write
\begin{align} 
	(1- \al \mu)^{T} &\leq \exp\left( -c_2 T^{1-\delta/2} \right) \eqqcolon e_1(T),  \nonumber\\
	\left( 1 - \frac{1}{2} \cdot \frac{\mu L}{\mu + L}  \al \right)^{T/2} &\leq \exp \left( -\frac{1}{2} \cdot \frac{\mu L}{\mu + L}c_2 T^{1-\delta/2} \right)\nonumber\\
	& \eqqcolon e_2(T).
\end{align}
Therefore, from (\ref{eq:boundT2}) we have
\begin{align} \label{eq:boundT3}
   \norm{\bx^*_{\al} - \bx^*}^2  &\leq  \frac{1}{\left( 1 - 4e_1(T)  \right)} \Bigg\{ 4e_1(T) n \norm{\widetilde{\bx}^*}^2  \nonumber \\
   &\quad + 2n e^2_2(T) \norm{\widetilde{\bx}^*}^2 \nonumber \\
   &\quad+ 4n e_2(T) \norm{\widetilde{\bx}^*} \frac{\al D}{1-\beta} \left(3+ 2L / \mu \right) \nonumber \\
   &\quad+ 2n D^2 \left(3+ 2L / \mu \right)^2 \left( \frac{\al}{1-\beta}  \right)^2 \Bigg\} \nonumber \\
   &\leq  \frac{4n \left( 2 e_1(T) + e^2_2(T) \right)}{\left( 1 - 4e_1(T)  \right)} \frac{f_0 - f^*}{\mu} \nonumber \\
   &\quad + \frac{4 \sqrt{2} n e_2(T)}{\left( 1 - 4e_1(T)  \right)} \sqrt{ \frac{f_0 - f^*}{\mu} } \frac{\al D}{1-\beta} \left(3+ 2L / \mu \right) \nonumber \\
   &\quad+ \frac{2n D^2 \left(3+ 2L / \mu \right)^2}{\left( 1 - 4e_1(T)  \right)} \left( \frac{\al}{1-\beta}  \right)^2,
\end{align}
where we used the fact that $\norm{\widetilde{\bx}^*}^2 \leq 2(f_0 - f^*)/ \mu$. Let $B_1(T)$ denote the bound in RHS of (\ref{eq:boundT3}). Given the fact that the terms $e_1(T)$ and $e_2(T)$ decay exponentially, i.e. $e_1(T)=o\left( \al^2 \right)$ and $e_2(T)=o\left( \al^2 \right)$,  we have 
\begin{align} \label{eq:boundT4}
   \norm{\bx^*_{\al} - \bx^*} &\leq \cO \left( \sqrt{2n} D \left(3+ 2L/\mu\right) \left( \frac{\al}{1-\beta}  \right) \right) \nonumber \\
   &=  \cO \left( \frac{\sqrt{2n} c_2 D \left(3+ 2L/\mu\right)}{ 1-\beta}  \frac{1}{T^{\delta/2}} \right)
\end{align}
which concludes the claim in Lemma \ref{lemma2}.
Moreover, due to the exponential decay of the two terms $e_1(T)$ and $e_2(T)$, we have 
\begin{align} 
   B_1(T) &\approx 2n D^2 \left( 3+ 2L / \mu \right)^2 \left( \frac{\al}{1-\beta}  \right)^2 \\
   & = \frac{2n c^2_2 D^2 \left(3+ 2L/\mu\right)^2 }{(1-\beta)^2}  \frac{1}{T^{\delta}}.
\end{align}


\section{Proof of Theorem \ref{thm2}}\label{app_proof_of_theorem_2}

Note that the steps of the proof are similar to the one for Theorem \ref{thm1}. There, we derived the convergence rate of each worker, i.e. $\bE \Big[\norm{\bx_{i,T} - \widetilde{\bx}^*}^2\Big]$ by bounding two quantities $\Expc \Big[\norm{\bx_T - \bx^*_{\al}}^2 \Big]$ and $\norm{\bx^*_{\al} - \bx^*}$ as in Lemma \ref{lemma1} and \ref{lemma2} respectively. Here, replacing Assumption \ref{assump3} by Assumption \ref{assump5} acquires only the former quantity to revisit. 
From (\ref{eq:bound1}), we have that for $t\geq 1$,
 \begin{align}
 \bE \left[ \norm{\bx_{t+1} - \bx^*_{\al}}^2 \vert \cF^t \right] &\leq (1-2 \mu_{\al} \eps)\norm{\bx_{t} - \bx^*_{\al}}^2 \nonumber\\
 & \quad + \eps^2 \bE \left[\norm{\widetilde{\gr}h_{\al}(\bx_t)}^2 | \cF^t\right]. \label{eq:newbound2}
\end{align}
 Considering Assumption \ref{assump5}, the second term in RHS of (\ref{eq:bound2}) can be bounded as follows, 
\begin{align}
\bE & \left[ \norm{\widetilde{\gr}h_{\al}(\bx_t)}^2 \vert \cF^t\right] \nonumber\\
&= \bE \left[ \norm{ \left(\bW_{D} - \bW\right) \bz_{t} + \left(\bI - \bW_D \right) \bx_t + \al \gr F(\bx_{t})}^2 \vert \cF^t\right] \nonumber\\
&\leq \norm{{\gr}h_{\al}(\bx_t)}^2 + \bE \left[ \norm{ \left(\bW_{D} - \bW\right)( \bz_{t} -\bx_t)}^2  \vert \cF^t\right]  \nonumber\\
&\leq L^2_{\al} \norm{\bx_{t} - \bx^*_{\al}}^2 + \eta^2 \norm{W-W_D}^2 \norm{\bx_{t}}^2 \nonumber\\
&= L^2_{\al} \norm{\bx_{t} - \bx^*_{\al}}^2 + \eta^2 \norm{W-W_D}^2 \norm{\bx_{t} - \bx^*_{\al} + \bx^*_{\al}}^2 \nonumber\\
& \leq \left( L^2_{\al} + 2 \eta^2 \norm{W - W_D} \right) \norm{\bx_{t} - \bx^*_{\al}}^2 \nonumber\\
& \quad + 2 \eta^2 \norm{W - W_D}^2 \norm{\bx^*_{\al}}^2. \label{eq:newg1}
\end{align}
Moreover, since the solution to Problem (\ref{eq:main}), i.e. $\norm{\widetilde{\bx}^*}$ (hence $\norm{{\bx}^*}$) is assumed to be bounded, the (unique) minimizer of $h_{\al}(\cdot)$, i.e. $\norm{{\bx}^*_{\al}}$ is also bounded as follows,
\begin{align}
\norm{\bx^*_{\al}}^2 &= \norm{\bx^*_{\al} - \bx^* + \bx^*}^2 \nonumber\\
& \leq 2\norm{\bx^*_{\al} - \bx^*}^2 + 2\norm{\bx^*}^2 \nonumber\\
&\leq 2 B_1(T) + \frac{4n(f_0 - f^*)}{\mu} \nonumber\\
&\leq 2 B_1(1) + \frac{4n(f_0 - f^*)}{\mu} \eqqcolon n \widetilde{B}^2. \label{eq:xstarb}
\end{align}
 Plugging (\ref{eq:newg1}) and (\ref{eq:xstarb}) into (\ref{eq:newbound2}) yields
 \begin{align}
 &\bE \left[\norm{\bx_{t+1} - \bx^*_{\al}}^2 | \cF^t \right] \nonumber\\
 &  \leq \left(1-2 \mu_{\al} \eps + \eps^2 \left(L^2_{\al} + + 2 \eta^2 \norm{W - W_D}^2 \right) \right)\norm{\bx_{t} - \bx^*_{\al}}^2 \nonumber\\
 & \quad + \eps^2 n \widetilde{B}^2 \norm{W-W_D}^2. \label{eq:newitr}
\end{align}
 Let us pick
\begin{align}
    \widetilde{T}_1 \coloneqq & \text{max} \Bigg\{ e^{e^{\frac{1}{1-2\delta}}},  \left\lceil \left( c_1 c_2 \mu \right)^{1/(2\delta)} \right\rceil , \nonumber\\
    & \left\lceil \left( \frac{c_1 ((2+c_2 L)^2 + 2 \eta^2 \norm{W-W_D}^2)}{c_2 \mu} \right)^{1/\delta} \right\rceil \Bigg\}.
\end{align}
For $ T \geq \widetilde{T}_1 $, we have 
\begin{align}
    \eps &= \frac{c_1}{T^{3\delta/2}}\nonumber\\
    &\leq \frac{c_2 \mu }{(2+c_2 L)^2 + 2 \eta^2 \norm{W-W_D}^2} \cdot \frac{1}{T^{\delta/2}}\nonumber\\
    &\leq \frac{\mu_{\al}}{ \big(1-\lambda_n(W) + \al L \big)^2 + 2 \eta^2 \norm{W-W_D}^2 }\nonumber\\
    &= \frac{\mu_{\al}}{L^2_{\al}+ 2 \eta^2 \norm{W-W_D}^2},
\end{align}
which together with (\ref{eq:newitr}) yields 
\begin{align}
    \bE \left[\norm{\bx_{t+1} - \bx^*_{\al}}^2 \right] &\leq (1- \mu_{\al} \eps ) \bE \left[\norm{\bx_{t+1} - \bx^*_{\al}}^2 \right] \nonumber\\
    & \quad + 2 \eps^2 n \widetilde{B}^2 \eta^2 \norm{W-W_D}^2.
\end{align}
Finally, from Lemma \ref{lemma3} with $a=c_1 c_2 \mu$ and $b=2c^2_1 {  n} \widetilde{B}^2 \eta^2 \norm{W-W_D}^2$, we have that
 \begin{align} \label{eq:boundTh2}
    \Expc \left[ \norm{\bx_T - \bx^*_{\al}}^2 \right]  
    &\leq \frac{2 c_1 {  n}\widetilde{B}^2 \eta^2 \norm{W-W_D}^2}{\mu c_2} \frac{1}{T^{\delta}}  \nonumber\\
    & \quad +  \exp\left( -c_1 c_2 \mu  T^{\delta}  \right) \sqrt{n} \widetilde{B}.
\end{align}
Let $B_2(T)$ denote the bound in RHS of (\ref{eq:boundTh2}). Due to the exponential decay of the second term in $B_2(T)$, we have 
 \begin{align}  \label{eq:bound2Th2}
    \Expc \left[ \norm{\bx_T - \bx^*_{\al}}^2 \right]  
    \leq \mathcal{O} \left( \frac{2 c_1 {  n}\widetilde{B}^2 \eta^2 \norm{W-W_D}^2}{\mu c_2} \frac{1}{T^{\delta}} \right),
\end{align}
and  
 \begin{align} 
    B_2(T) \approx \frac{2 c_1 {  n}\widetilde{B}^2 \eta^2 \norm{W-W_D}^2}{\mu c_2} \frac{1}{T^{\delta}}.
\end{align}
Hence, by putting (\ref{eq:bound2Th2}) together with Lemma \ref{lemma2} we conclude the claim for any $T \geq \widetilde{T}_0 \coloneqq \text{max} \left\{ \widetilde{T}_1, T_2 \right\}$.

\end{appendices}

\bibliographystyle{ieeetr}
\bibliography{biblio.bib}

\begin{IEEEbiography}[{\includegraphics[width=1in,height=1.25in,clip,keepaspectratio]{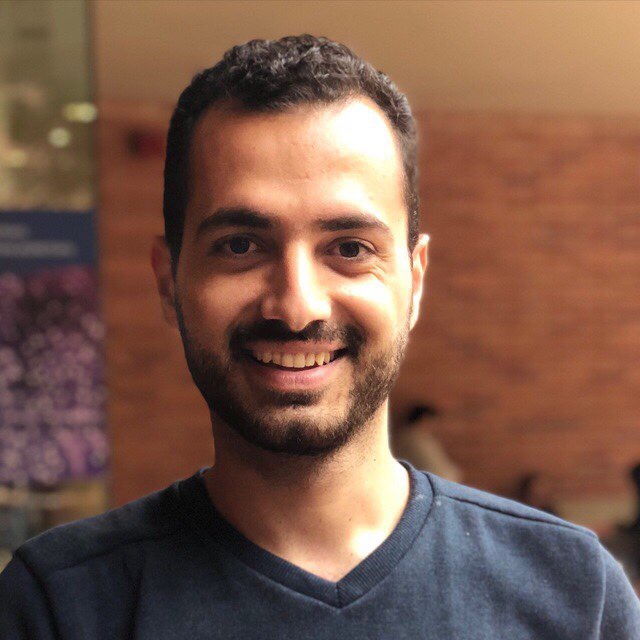}}]{Amirhossein Reisizadeh}
received his B.S. degree form Sharif University of Technology, Tehran, Iran in 2014 and an M.S. degree from University of California, Los Angeles (UCLA) in 2016, both in Electrical Engineering. He is currently a Ph.D. candidate in the Department of  Electrical and Computer Engineering at University of California, Santa Barbara (UCSB). He is interested in using information and coding-theoretic concepts to develop fast and efficient algorithms for large-scale machine learning, distributed computing and optimization. He was a finalist for the Qualcomm Innovation Fellowship.
\end{IEEEbiography}

\begin{IEEEbiography}[{\includegraphics[width=1in,height=1.25in,clip,keepaspectratio]{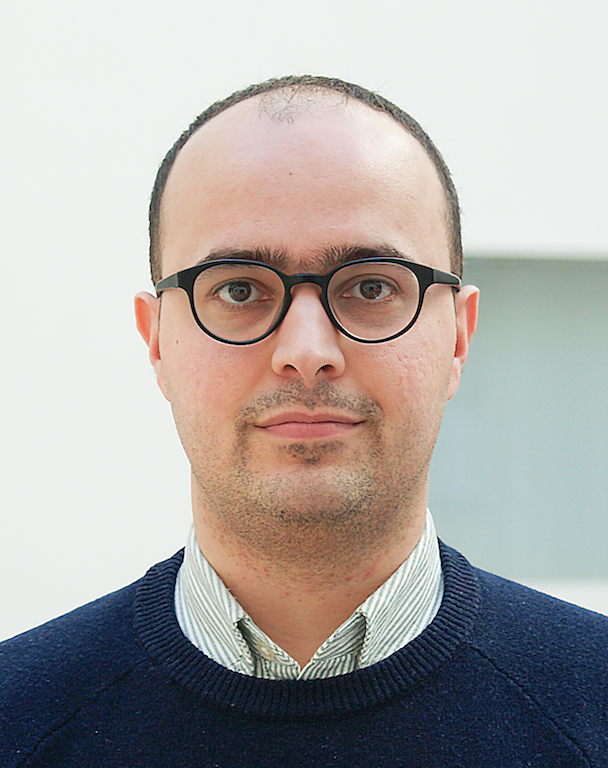}}]{Aryan Mokhtari}
 received the B.Sc. degree in electrical engineering from Sharif University of Technology, Tehran, Iran, in 2011, and the M.Sc. and Ph.D. degrees in electrical and systems engineering from the University of Pennsylvania (Penn), Philadelphia, PA, USA, in 2014 and 2017, respectively. He also received his A.M. degree in statistics from the Wharton School at Penn in 2017. He is currently an Assistant Professor in the Department of Electrical and Computer Engineering at the University of Texas at Austin, Austin, TX, USA. Prior to that, he was a Postdoctoral Associate in the Laboratory for Information and Decision Systems (LIDS) at the Massachusetts Institute of Technology (MIT), Cambridge, MA, USA, from January 2018 to July 2019. Before joining MIT, he was a Research Fellow at the Simons Institute for the Theory of Computing at the University of California, Berkeley, for the program on “Bridging Continuous and Discrete Optimization”, from August to December 2017. His research interests include the areas of optimization, machine learning, and signal processing. His current research focuses on the theory and applications of convex and non-convex optimization in large-scale machine learning and data science problems. He has received a number of awards and fellowships, including Penn’s Joseph and Rosaline Wolf Award for Best Doctoral Dissertation in electrical and systems engineering and the Simons-Berkeley Fellowship.
\end{IEEEbiography}

\begin{IEEEbiography}[{\includegraphics[width=1in,height=1.25in,clip,keepaspectratio]{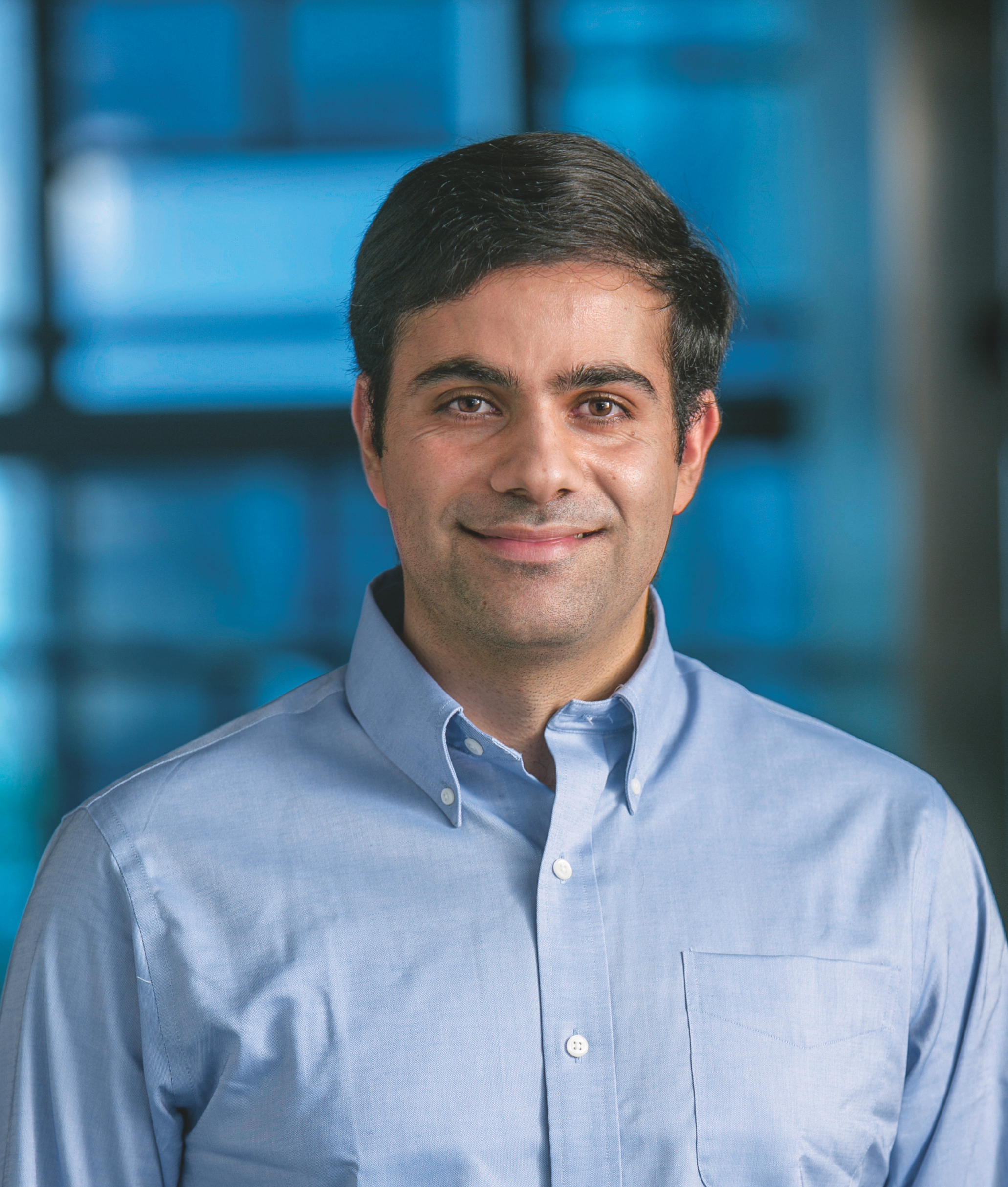}}]{Hamed Hassani}
 (IEEE member since 2010) is an assistant professor in the Department of Electrical and Systems Engineering (ESE) at the University of Pennsylvania. Prior to that, he was a research fellow at the Simons Institute at UC Berkeley, and a post-doctoral scholar in the Institute for Machine Learning at ETH Zurich. He obtained his Ph.D. degree in Computer and Communication Sciences from EPFL. For his PhD thesis he received the 2014 IEEE Information Theory Society Thomas M. Cover Dissertation Award. He also received the Jack K. Wolf Student paper award from the 2015 IEEE International Symposium on Information Theory (ISIT).  He has B.Sc. degrees in Electrical Engineering and Mathematics from Sharif University of Technology, Iran.
\end{IEEEbiography}

\begin{IEEEbiography}[{\includegraphics[width=1in,height=1.25in,clip,keepaspectratio]{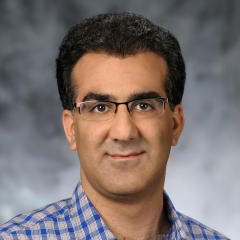}}]{Ramtin Pedarsani}
is an Assistant Professor in ECE Department at the University of California, Santa Barbara. He received the B.Sc. degree in electrical engineering from the University of Tehran, Tehran, Iran, in 2009, the M.Sc. degree in communication systems from the Swiss Federal Institute of Technology (EPFL), Lausanne, Switzerland, in 2011, and his Ph.D. from the University of California, Berkeley, in 2015. His research interests include machine learning, information and coding theory, networks, and transportation systems. Ramtin is a recipient of the IEEE international conference on communications (ICC) best paper award in 2014.
\end{IEEEbiography}

 \end{document}